\pdfoutput=1

\documentclass[11pt]{article}

\usepackage{acl}

\usepackage{times}
\usepackage{latexsym}

\usepackage{graphicx}
\usepackage{subfigure}
\usepackage{bm}
\usepackage{amssymb}
\usepackage{amsmath}
\usepackage{amsthm}
\usepackage{xcolor}
\usepackage{multirow}
\usepackage{enumitem}

\newtheorem{theorem}{Theorem}

\usepackage{pifont}
\newcommand{\cmark}{\ding{51}}
\newcommand{\xmark}{\ding{55}}

\usepackage[T1]{fontenc}

\usepackage[utf8]{inputenc}

\usepackage{microtype}

\title{ASSIST: Towards Label Noise-Robust Dialogue State Tracking}

\author{Fanghua Ye \and Yue Feng \and Emine Yilmaz \\
  University College London\\
  London, UK \\
  
  \texttt{\{fanghua.ye.19, yue.feng.20, emine.yilmaz\}@ucl.ac.uk} }

\begin{document}
\maketitle
\begin{abstract}
The MultiWOZ 2.0 dataset has greatly boosted the research on dialogue state tracking (DST). However, substantial noise has been discovered in its state annotations. Such noise brings about huge challenges for training DST models robustly. Although several refined versions, including MultiWOZ 2.1-2.4, have been published recently, there are still lots of noisy labels, especially in the training set. Besides, it is costly to rectify all the problematic annotations. In this paper, instead of improving the annotation quality further, we propose a general framework, named ASSIST (l\textbf{A}bel noi\textbf{S}e-robu\textbf{S}t d\textbf{I}alogue \textbf{S}tate \textbf{T}racking), to train DST models robustly from noisy labels. ASSIST first generates pseudo labels for each sample in the training set by using an auxiliary model trained on a small clean dataset, then puts the generated pseudo labels and vanilla noisy labels together to train the primary model. We show the validity of ASSIST theoretically. Experimental results also demonstrate that ASSIST improves the joint goal accuracy of DST by up to $28.16\%$ on MultiWOZ 2.0 and $8.41\%$ on  MultiWOZ 2.4, compared to using only the vanilla noisy labels. 
\end{abstract}

\section{Introduction}

Task-oriented dialogue systems play an important role in helping users accomplish a variety of tasks through verbal interactions~\citep{young2013pomdp, gao2019neural}. Dialogue state tracking (DST) is an essential component of the dialogue manager in pipeline-based task-oriented dialogue systems. It aims to keep track of users' intentions at each turn of the conversation~\citep{mrksic-etal-2017-neural}. The state information indicates the progress of the conversation and is leveraged to determine the next system action and generate the next system response~\citep{chen-etal-2017-deep}. As shown in Figure~\ref{fig:example}, the dialogue state is typically represented as a set of (\textit{slot}, \textit{value}) pairs~\citep{williams2014dialog, henderson-etal-2014-second}. 
Therefore, the problem of DST is defined as extracting the values for all slots from the dialogue context at each turn of the conversation.

Over the past few years, DST has made significant progress, attributed to a number of publicly available dialogue datasets, such as DSTC2~\citep{henderson-etal-2014-second}, FRAMES~\citep{el-asri-etal-2017-frames}, MultiWOZ 2.0~\citep{budzianowski-etal-2018-multiwoz}, CrossWOZ~\citep{zhu2020crosswoz}, and SGD~\citep{rastogi2020towards}. Among these datasets, MultiWOZ 2.0 is the most popular one. So far, lots of DST models have been built on top of it~\citep{lee-etal-2019-sumbt, wu-etal-2019-transferable, ouyang-etal-2020-dialogue, kim-etal-2020-efficient, hu-etal-2020-sas, ye2021slot, lin2021knowledge}.

\begin{figure}[t]
  \centering
  \includegraphics[width=1.0\linewidth]{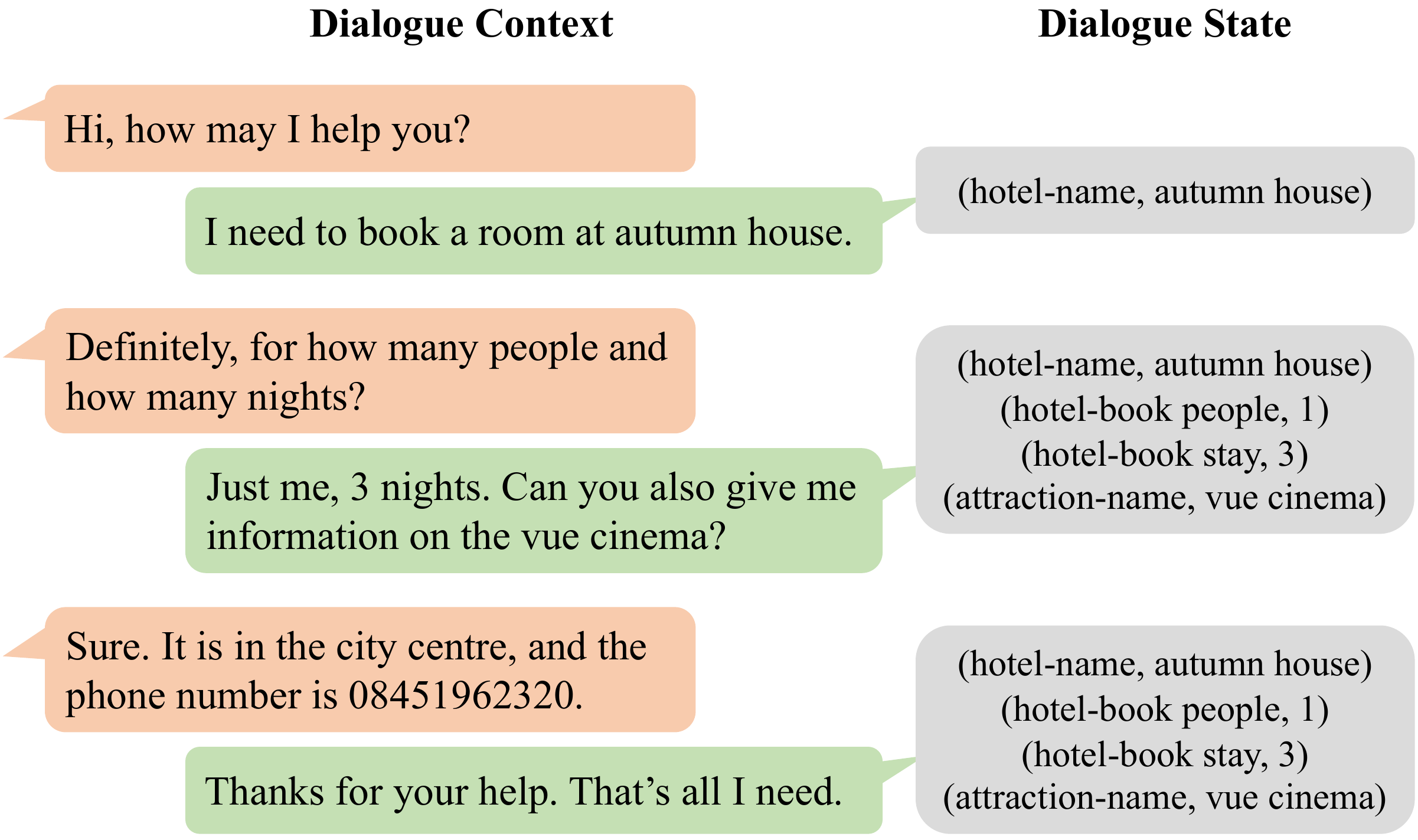}
  \caption{An example dialogue spanning two domains. On the left is the dialogue context with system respon-ses shown in orange and user utterances in green. The dialogue state at each turn is presented on the right.}
  \label{fig:example}
\end{figure}

However, it has been found out that there is substantial noise in the state annotations of MultiWOZ 2.0~\citep{eric-etal-2020-multiwoz}. These noisy labels may impede the training of robust DST models and lead to noticeable performance decrease~\citep{zhang2016understanding}. To remedy this issue, massive efforts have been devoted to rectifying the annotations, and four refined versions, including MultiWOZ 2.1~\citep{eric-etal-2020-multiwoz}, MultiWOZ 2.2~\citep{zang2020multiwoz}, MultiWOZ 2.3~\citep{han2020multiwoz}, and MultiWOZ 2.4~\citep{ye2021multiwoz}, have been released. Even so, there are still plenty of noisy and inconsistent labels. For example, in the latest version MultiWOZ 2.4, the validation set and test set have been manually re-annotated and tend to be noise-free. While the training set is still noisy, as it remains intact. In reality, it is costly and laborious to refine existing large-scale noisy datasets or collect new ones with fully precise annotations~\citep{wei2020combating}, let alone dialogue datasets with multiple domains and multiple turns. In view of this, we argue that it is essential to devise particular learning algorithms to train DST models robustly from noisy labels.

Although loads of noisy label learning algorithms~\citep{natarajan2013learning, han2020survey} have been proposed in the machine learning community, most of them target only multi-class classification~\citep{song2020learning}. However, as illustrated in Figure~\ref{fig:example}, the dialogue state may contain multiple labels, which makes it unstraightforward to apply existing noisy label learning algorithms to the DST task. In this paper we propose a general framework, named ASSIST (l\textbf{A}bel noi\textbf{S}e-robu\textbf{S}t d\textbf{I}alogue \textbf{S}tate \textbf{T}racking), to train DST models robustly from noisy labels. ASSIST first trains an auxiliary model on a small clean dataset to generate pseudo labels for each sample in the noisy training set. Then, it leverages both the generated pseudo labels and vanilla noisy labels to train the primary model. Since the auxiliary model is trained on the clean dataset, it can be expected that the pseudo labels will help us train the primary model more robustly. Note that ASSIST is based on the assumption that we have access to a small clean dataset. This assumption is reasonable, as it is feasible to manually collect a small noise-free dataset or re-annotate a portion of a large noisy dataset.

In summary, our main contributions include:
\begin{itemize}
    \item We propose a general framework ASSIST to train robust DST models from noisy labels. To the best of our knowledge, we are the first to tackle the DST problem by taking into consideration the label noise.
    
    \item We theoretically analyze why the pseudo labels are beneficial and show that a proper combination of the pseudo labels and vanilla noisy labels can approximate the unknown true labels more accurately.
    
    \item We conduct extensive experiments on MultiWOZ 2.0 \& 2.4. The results demonstrate that ASSIST can improve the DST performance on both datasets by a large margin.
\end{itemize}

\section{Problem Definition}

In this section, we first provide the conventional definition of DST and then extend the definition to the noisy label learning scenario.

\subsection{Conventional Dialogue State Tracking}

Let $\mathcal{X}=\{(R_1, U_1), \dots, (R_T, U_T)\}$ denote a dialogue of $T$ turns, where $R_t$ and $U_t$ represent the system response and user utterance at turn $t$, respectively. The dialogue state at turn $t$ is defined as $\mathcal{B}_t = \{(s, v_t)|s\in \mathcal{S}\}$, where $\mathcal{S}$ denotes the set of predefined slots and $v_t$ is the corresponding value of slot $s$. Following previous work~\citep{lee-etal-2019-sumbt, hu-etal-2020-sas, ye2021slot}, a slot in this paper refers to the concatenation of the domain name and slot name so as to include the domain information. For example, we use "\textit{hotel-name}" to represent the slot "\textit{name}" in the hotel domain.

In general, the issue of DST is defined as learning a dialogue state tracker $\mathcal{F}: \mathcal{X}_t \rightarrow \mathcal{B}_t$ that takes the dialogue context $\mathcal{X}_t$ as input and predicts the dialogue state $\mathcal{B}_t$ at each turn $t$ as accurately as possible. Here, $\mathcal{X}_t$ represents the dialogue history up to turn $t$, i.e., $\mathcal{X}_t = \{(R_1, U_1), \dots, (R_t, U_t)\}$.

\subsection{Dialogue State Tracking with Noisy Labels}

Conventionally, all the state labels are assumed to be correct. However, this assumption may not hold.  In practice, dialogue state annotations are error-prone~\citep{han2020multiwoz}. There are a couple of reasons. First, the states are usually annotated by crowdworkers to improve the labelling efficiency. Due to limited knowledge, crowdworkers cannot annotate all the states with $100\%$ accuracy, which naturally incurs noisy labels~\citep{han2020survey}. Second, the dialogue may span multiple domains, which also increases the labelling difficulty. Apparently, the noisy labels are harmful and likely to lead to sub-optimal performance. Therefore, it is crucial to take them into consideration so as to train DST models more robustly.

Let $\tilde{\mathcal{B}}_t = \{(s, \tilde{v}_t)|s \in \mathcal{S}\}$ denote the noisy state annotations, where $\tilde{v}_t$ is the noisy label of slot $s$ at turn $t$. We use $\mathcal{B}_t = \{(s, v_t)|s\in S\}$ to denote the noise-free state annotations. Here, $v_t$ represents the true label of slot $s$ at turn $t$, which is unknown. In fact, existing DST approaches are only able to learn a sub-optimal dialogue state tracker $\tilde{\mathcal{F}}: \mathcal{X}_t \rightarrow \tilde{\mathcal{B}}_t$ rather than the optimal state tracker $\mathcal{F}: \mathcal{X}_t \rightarrow \mathcal{B}_t$, as none of them have considered the influence of noisy labels. In this work, we aim to learn a robust state tracker $\mathcal{F}^*$ that can better approximate $\mathcal{F}$ from the noisy state annotations $\tilde{\mathcal{B}}_t$.

\section{Proposed Approach}

We introduce a general framework ASSIST, aiming to train DST models robustly from noisy labels. We assume that a small clean dataset is accessible. Based on this dataset, ASSIST first trains an auxiliary model $\mathcal{A}$. Then, it leverages $\mathcal{A}$ to generate pseudo labels for each sample in the noisy training set. The pseudo state annotations are represented as $\breve{\mathcal{B}}_t = \{(s, \breve{v}_t)|s \in \mathcal{S}\}$, where $\breve{v}_t$ denotes the pseudo label of slot $s$ at turn $t$. Afterwards, both the generated pseudo labels and vanilla noisy labels are exploited to train the primary model $\mathcal{F}^*$. That is, we intend to learn $\mathcal{F}^*: \mathcal{X}_t \rightarrow C(\breve{\mathcal{B}}_t, \tilde{\mathcal{B}}_t)$, where $C(\breve{\mathcal{B}}_t, \tilde{\mathcal{B}}_t)$ is a combination of $\breve{\mathcal{B}}_t$ and $\tilde{\mathcal{B}}_t$. 

Essentially, any existing DST models can be employed as the auxiliary model. However, these models may lead to overfitting due to the small size of the clean dataset. To tackle this issue, we propose a new simple model as the auxiliary model\footnote{We adopt existing DST models as the primary model.}.

\subsection{Auxiliary Model Architecture}

Figure~\ref{fig:architecture} shows the architecture, which consists of a dialogue context semantic encoder, a slot attention module, and a slot-value matching module.

\subsubsection*{Dialogue Context Semantic Encoder}

Similar to \citep{lee-etal-2019-sumbt, kim-etal-2020-efficient, ye2021slot}, we utilize the pre-trained language model BERT~\citep{devlin-etal-2019-bert} to encode the dialogue context $\mathcal{X}_t$ into contextual semantic representations. Let $ {Z}_t = R_t \oplus U_t$ be the concatenation of the system response and user utterance at turn $t$, where $\oplus$ denotes the operator of sequence concatenation. Then, the dialogue context $\mathcal{X}_t$ can be represented as $X_t = {Z}_1 \oplus {Z}_2 \oplus \cdots \oplus {Z}_t$.

We also concatenate each slot-value pair and denote the representation of the dialogue state at turn $t$ as ${B}_t = \bigoplus_{(s, v_{t}) \in {\mathcal{B}}_{t}, v_{t} \neq \verb|none|} s \oplus v_{t}$, in which only non-\verb|none| slots are included. $B_t$ can serve as a compact representation of the dialogue history. In view of this, we treat the previous turn dialogue state $B_{t-1}$ as part of the input as well, which can be beneficial when $X_t$ exceeds the maximum input length of BERT. The complete input sequence to the encoder module is then denoted as:
\begin{equation*}
    I_t = [CLS] \oplus X_{t-1} \oplus B_{t-1} \oplus [SEP] \oplus {Z}_t \oplus [SEP],
\end{equation*}
where $[CLS]$ and $[SEP]$ are the two special tokens introduced by BERT.

Let $\bm{H}_t  \in \mathbb{R}^{|I_t| \times d}$ be the semantic matrix representation of $I_t$. Here, $|I_t|$ and $d$ denote the sequence length of $I_t$ and the BERT output dimension, respectively. Then, we have: 
\begin{equation*} 
    \bm{H}_t = \verb|BERT|_{finetune}(I_t),
\end{equation*}
where $\verb|BERT|_{finetune}$ means that the BERT model will be fine-tuned during the training process.

For each slot $s$ and its candidate value $v' \in \mathcal{V}_s$, we employ another BERT to encode them into semantic vectors $\bm{h}^s \in \mathbb{R}^d$ and $\bm{h}^{v'} \in \mathbb{R}^d$. Here, $\mathcal{V}_s$ denotes the candidate value set of slot $s$. Unlike the dialogue context, we leverage the pre-trained BERT without fine-tuning to embed $s$ and $v'$. Besides, we adopt the output vector corresponding to the special token $[CLS]$ as an aggregated representation of slot $s$ and value $v'$, i.e.,
\begin{equation*}
    \begin{aligned}
        \bm{h}^s &= \verb|BERT|^{[CLS]}_{fixed}([CLS] \oplus s \oplus [SEP]), \\
        \bm{h}^{v'} &= \verb|BERT|^{[CLS]}_{fixed}([CLS] \oplus v' \oplus [SEP]).
    \end{aligned}
\end{equation*}

\begin{figure}[t]
  \centering
  \includegraphics[width=0.97\linewidth]{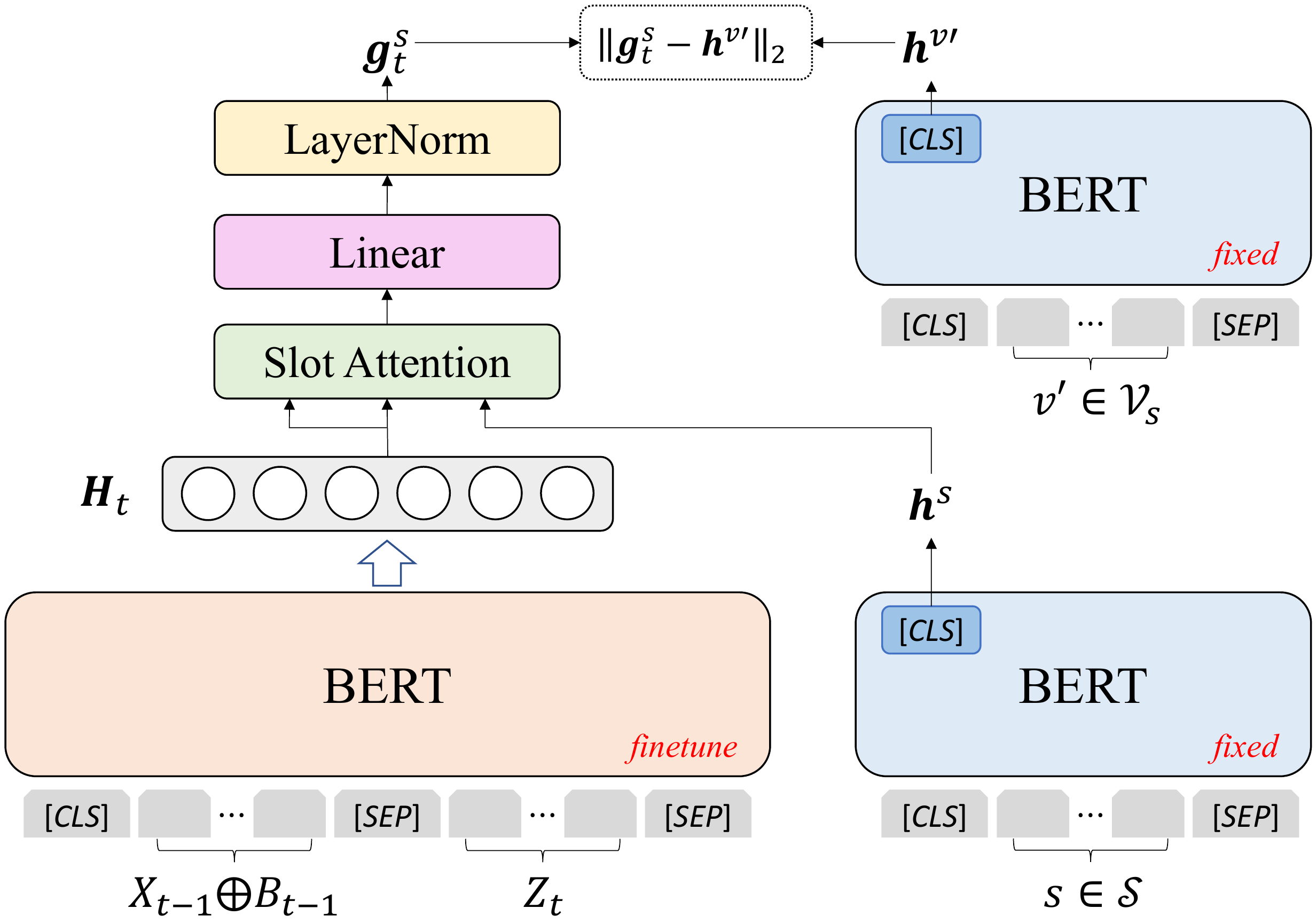}
  \caption{Overall architecture of the auxiliary model. The parameters of the BERT used to encode slots and values are fixed during the training process.}
  \label{fig:architecture}
\end{figure}

\subsubsection*{Slot Attention}

The slot attention module is exploited to retrieve slot-relevant information for all the slots from the same dialogue context. The slot attention is a multi-head attention~\citep{vaswani2017attention}. Specifically, the slot representation $\bm{h}^s$ is regarded as the query vector, and the dialogue context representation $\bm{H}_t$ is taken as both the key matrix and value matrix. The slot attention matches $\bm{h}^s$ to the semantic vector of each word in the dialogue context and calculates the attention score, based on which the slot-specific information can be extracted. Let $\bm{a}^s_t \in \mathbb{R}^d$ denote a $d$-dimensional vector  representation of the related information of slot $s$ at turn $t$, we obtain: 
\begin{equation*}
    \bm{a}^s_t = \verb|MultiHead|(\bm{h}^s, \bm{H}_t, \bm{H}_t).
\end{equation*}
$\bm{a}^s_t$ is expected to be close to the semantic vector representation of the true value of slot $s$. 

Considering that the output of BERT is normalized by layer normalization~\citep{ba2016layer}, we also feed $\bm{a}^s_t$ to a layer normalization layer, which is preceded by a linear transformation layer. The final slot-specific vector $\bm{g}^s_t \in \mathbb{R}^d$ is calculated as:
\begin{equation*}
   \bm{g}^s_t = \verb|LayerNorm|(\verb|Linear|(\bm{a}^s_t)).
\end{equation*}

\subsubsection*{Slot-Value Matching}

The slot-value matching module is utilized to predict the value of each slot $s$. It first calculates the distance between the slot-specific representation $\bm{g}^s_t$ and the semantic representation of each candidate value $v' \in \mathcal{V}_s$, i.e., $\bm{h}^{v'}$. Then, the candidate value with the smallest distance is selected as the prediction. The $\ell_2$ norm is adopted to compute the distance. Denoting $\hat{v}_t$ as the predicted value of slot $s$ at turn $t$, we have:
\begin{equation*}
    \hat{v}_t = \operatorname*{argmin}_{v' \in \mathcal{V}_s} \Vert \bm{g}^s_t - \bm{h}^{v'} \Vert_2.
\end{equation*}

\subsection{Auxiliary Model Training}

We leverage a small clean dataset to train the auxiliary model. Since the true labels are available, the auxiliary model is directly trained to maximize the joint probability of all slot values. The probability of the true value $v_t$ of slot $s$ at turn $t$ is defined as:
\begin{equation*}
    p(v_t|\mathcal{X}_t, s) = \frac{\exp\left(-\Vert \bm{g}^s_t - \bm{h}_t^{v} \Vert_2\right)}{\sum_{v' \in \mathcal{V}_s} \exp\left(-\Vert \bm{g}^s_t - \bm{h}^{v'} \Vert_2\right)},
\end{equation*}
where $\bm{h}^v_t$ is the semantic representation of $v_t$. Maximizing the joint probability $\Pi_{(s, v_t) \in \mathcal{B}_t} p(v_t|\mathcal{X}_t, s)$ is equivalent to minimizing the following objective:
\begin{equation*}
    \mathcal{L}_{aux} = \sum_{(s, v_t) \in \mathcal{B}_t} - \log p(v_t|\mathcal{X}_t, s).
\end{equation*}

\subsection{Pseudo Label Generation}
Our approach depends on the auxiliary model $\mathcal{A}$ to generate pseudo labels $\breve{\mathcal{B}}_t = \{(s, \breve{v}_t)|s \in \mathcal{S}\}$ for each sample in the noisy training set. In this work, we treat each dialogue context $\mathcal{X}_t$ rather than the entire dialogue as a training sample. Without loss of generality, the pseudo label generation process is denoted as follows:
\begin{equation*}
    \breve{\mathcal{B}}_t = \mathcal{A}(\mathcal{X}_t, \mathcal{S}),
\end{equation*}
where $\mathcal{X}_t$ belongs to the noisy training set.

\subsection{Primary Model Training}

To reduce the influence of noisy labels, we combine the generated pseudo labels and vanilla noisy labels to train the primary model. 

Let $\breve{\bm{v}}_t$ and  $\tilde{\bm{v}}_t$ be the one-hot representation of the pseudo label $\breve{v}_t$ and vanilla noisy label $\tilde{v}_t$, respectively. Then, we can define the combined label as:
\begin{equation*}
    \bm{v}^c_t = \alpha \breve{\bm{v}}_t + (1 - \alpha) \tilde{\bm{v}}_t,
\end{equation*}
where $\alpha (0 \le \alpha \le 1)$ is a parameter to balance the pseudo labels and vanilla labels. We calculate the probability of $\bm{v}^c_t$ as below: 
\begin{equation*}
    p(\bm{v}^c_t|\mathcal{X}_t, s) = p(\breve{v}_t | \mathcal{X}_t, s)^{\alpha}  p(\tilde{v}_t | \mathcal{X}_t, s)^{(1 - \alpha)}.
\end{equation*}
Here, $p(\breve{v}_t | \mathcal{X}_t, s)$ and $p(\tilde{v}_t | \mathcal{X}_t, s)$ correspond to the probability of $\breve{v}_t$ and $\tilde{v}_t$, respectively.

Let $C(\breve{\mathcal{B}}_t, \tilde{\mathcal{B}}_t) = \{(s, \bm{v}^c_t)|s \in \mathcal{S}\}$ represent the combined state annotations. The training objective of the primary model is then defined as:
\begin{equation*}
\begin{aligned}
    \mathcal{L}_{pri} =& \sum_{(s, \bm{v}^c_t) \in C(\breve{\mathcal{B}}_t, \tilde{\mathcal{B}}_t)} - \log p(\bm{v}^c_t|\mathcal{X}_t, s) \\
    =& ~ \alpha \sum_{(s, \breve{v}_t) \in \breve{\mathcal{B}}_t} - \log p(\breve{v}_t | \mathcal{X}_t, s) \\
    &+ (1-\alpha) \sum_{(s, \tilde{v}_t) \in \tilde{\mathcal{B}}_t} -\log p(\tilde{v}_t | \mathcal{X}_t, s) \\
    = & ~ \alpha \mathcal{L}_{pseudo} + (1- \alpha) \mathcal{L}_{vanilla},
\end{aligned}
\end{equation*}
where $\mathcal{L}_{pseudo}$ and $\mathcal{L}_{vanilla}$ correspond to the training objective of using only the pseudo labels and using only the vanilla noisy labels, respectively. By minimizing $ \mathcal{L}_{pri}$, the primary model is trained to learn from the vanilla noisy labels and at the same time imitate the predictions of the auxiliary model.

\subsection{Theoretical Analysis}

Since the pseudo labels are generated by the auxiliary model that has been trained on a small clean dataset, it can be expected that the combined labels are able to serve as a better approximation to the unknown true labels. Let $\bm{v}_t$ denote the one-hot representation of the unknown true value $v_t$ of slot $s$ at turn $t$. We adopt the mean squared loss to define the approximation error of any corrupted labels $\ddot{\bm{v}}_t$ associated with the noisy training set $\mathcal{D}_n$ as:
\begin{equation*}
     Y_{\ddot{\bm{v}}} = \frac{1}{|\mathcal{D}_n||\mathcal{S}|}\sum_{\mathcal{X}_t \in \mathcal{D}_n} \sum_{s \in \mathcal{S}} E_{\mathcal{D}_c} [ \Vert \ddot{\bm{v}}_t - \bm{v}_t \Vert^2_2],
\end{equation*}
where the expectation ranges over different choices of the clean dataset $\mathcal{D}_c$, and $|\cdot|$ returns the cardinality of a set.


Next, we show that the approximation error of the combined labels can be smaller than that of both the vanilla noisy labels and the generated pseudo labels. The details are presented in Theorem~\ref{thm}.

\begin{theorem} \label{thm}
 The optimal approximation error with respect to the combined labels $\bm{v}^c_t$ is smaller than that of the vanilla labels $\tilde{\bm{v}}_t$ and pseudo labels $\breve{\bm{v}}_t$, i.e., 
 \begin{equation*}
     \min_{\alpha} Y_{\bm{v}^c} < \min \{Y_{\tilde{\bm{v}}}, Y_{\breve{\bm{v}}}\}.
 \end{equation*}
 By setting $\alpha = \frac{Y_{\tilde{\bm{v}}}}{Y_{\tilde{\bm{v}}} + Y_{\breve{\bm{v}}}}$, $Y_{\bm{v}^c}$ reaches its minimum:
 \begin{equation*}
     \min_{\alpha} Y_{\bm{v}^c} = \frac{Y_{\tilde{\bm{v}}}  Y_{\breve{\bm{v}}}}{Y_{\tilde{\bm{v}}} + Y_{\breve{\bm{v}}}}.
 \end{equation*}
\end{theorem}
\begin{proof}
The proof is presented in Appendix~\ref{sec:appendixproof}. 
\end{proof}

Theorem~\ref{thm} indicates that if $\alpha$ is set properly, the combined labels can approximate the unknown true labels more accurately. Hence, we can potentially train the primary model more robustly. Note that we cannot calculate the optimal value of $\alpha$ directly.

\section{Experimental Setup}

\subsection{Datasets}

We adopt MultiWOZ 2.0~\citep{budzianowski-etal-2018-multiwoz} and MultiWOZ 2.4~\citep{ye2021multiwoz} as the datasets in our experiments. MultiWOZ 2.0 is one of the largest publicly available multi-domain task-oriented dialogue datasets, including about 10,000 dialogues spanning seven domains. MultiWOZ 2.4 is the latest refined version of MultiWOZ 2.0.  The annotations of its validation set and test set have been manually rectified. While its training set remains intact and is the same as that of MultiWOZ 2.1~\citep{eric-etal-2020-multiwoz}, in which $41.34\%$ of the state values are changed, compared to MultiWOZ 2.0.

Since the \textit{hospital} domain and \textit{police} domain never occur in the test set, we use only the remaining five domains $\{$\textit{attraction}, \textit{hotel}, \textit{restaurant}, \textit{taxi}, \textit{train}$\}$ in our experiments. These domains have 30 slots in total. Considering that the validation set and test set of MultiWOZ 2.0 are noisy, we replace them with the \textit{counterparts} of MultiWOZ 2.4\footnote{Despite this change, we still call the dataset MultiWOZ 2.0 in this paper for ease of exposition.}. We preprocess the datasets following~\citep{ye2021slot}. We use the validation set as the small clean dataset.

\subsection{Evaluation Metrics}

We exploit joint goal accuracy and slot accuracy as the evaluation metrics. The joint goal accuracy is defined as the proportion of dialogue turns in which the values of all slots are correctly predicted. It is the most important metric in the DST task. The slot accuracy is defined as the average of all individual slot accuracies. The accuracy of an individual slot is calculated as the ratio of dialogue turns in which its value is correctly predicted.

We also propose a new evaluation metric, termed as joint turn accuracy. We define joint turn accuracy as the proportion of dialogue turns in which the values of all \textit{active} slots are correctly predicted. A slot becomes active if its value is mentioned in current turn and is not inherited from previous turns. The advantage of joint turn accuracy is that it can tell us in how many turns the turn-level information is fully captured by the model.

\subsection{Primary DST Models}

To verify the effectiveness of the proposed framework, we apply the generated pseudo labels to three different primary models.
\begin{itemize}[label={}, leftmargin=*]
    \item \textbf{SOM-DST:} SOM-DST~\citep{kim-etal-2020-efficient} is an open vocabulary-based method. It treats the dialogue state as an explicit fixed-sized memory and selectively overwrites this memory at each turn.
    
    \item \textbf{STAR:} STAR~\citep{ye2021slot} is a predefined ontology-based method. It leverages a stacked slot self-attention mechanism to capture the slot dependencies automatically.
    
    \item \textbf{AUX-DST:} We also test using the proposed auxiliary model as the primary model. For the sake of description, we refer to this model as AUX-DST.
\end{itemize}

\begin{table*}[t]
\centering
\setlength{\tabcolsep}{1.98mm}
\begin{tabular}{l|cc|ccc|ccc}
\hline
\multirow{3}{*}{\begin{tabular}[c]{@{}l@{}}\textbf{Primary} \\ \textbf{Models}\end{tabular}} & \multicolumn{2}{c|}{\textbf{Labels}} &  \multicolumn{3}{c|}{\textbf{MultiWOZ 2.0}} & \multicolumn{3}{c}{\textbf{MultiWOZ 2.4}} \\ \cline{2-9} 
 & \textbf{Vanilla} & \textbf{Pseudo} & \begin{tabular}[c]{@{}c@{}}\textbf{Joint} \\ \textbf{Goal(\%)}\end{tabular} & \begin{tabular}[c]{@{}c@{}}\textbf{Joint} \\ \textbf{Turn(\%)}\end{tabular} & \textbf{Slot(\%)} & \begin{tabular}[c]{@{}c@{}}\textbf{Joint} \\ \textbf{Goal(\%)}\end{tabular} & \begin{tabular}[c]{@{}c@{}}\textbf{Joint} \\ \textbf{Turn(\%)}\end{tabular} & \textbf{Slot(\%)} \\ \hline
\multirow{3}{*}{SOM-DST} & \cmark & \xmark & 45.14 & 77.86 & 96.71 & 66.78 & 87.81 & 98.38 \\
 & \xmark & \cmark & 67.06 & 87.95 & 98.47 & 68.69 & 88.41 & 98.55 \\
 & \cmark & \cmark & \textbf{70.83} & \textbf{89.14} & \textbf{98.61} & \textbf{75.19} & \textbf{91.02} & \textbf{98.84} \\ \hline
\multirow{3}{*}{STAR} & \cmark & \xmark & 48.30 & 78.91 & 97.10 & 73.62 & 90.45 & 98.85 \\
 & \xmark & \cmark & 70.66 & 85.93 & 98.67 & 71.01 & 86.31 & 98.69 \\
 & \cmark & \cmark & \textbf{74.12} & \textbf{88.93} & \textbf{98.86} & \textbf{79.41} & \textbf{91.86} & \textbf{99.14} \\ \hline
\multirow{3}{*}{AUX-DST} & \cmark & \xmark & 45.66 & 78.76 & 96.95 & 70.37 & 89.31 & 98.67 \\
 & \xmark & \cmark & 70.39 & 86.28 & 98.67 & 70.68 & 86.82 & 98.68 \\
 & \cmark & \cmark & \textbf{73.82} & \textbf{88.29} & \textbf{98.84} & \textbf{78.14} & \textbf{91.03} & \textbf{99.07} \\ \hline
\end{tabular}
\caption{Performance comparison on MultiWOZ 2.0 and MultiWOZ 2.4. Note that MultiWOZ 2.0 and MultiWOZ 2.4 share the same test set in our experiments. The best scores are highlighted in bold.}
\label{tab:allmetrics}
\end{table*}

\subsection{Implementation Details}

For the auxiliary model, the pre-trained BERT-base-uncased model is utilized as the dialogue context encoder. Another pre-trained BERT-base-uncased model with fixed weights is employed to encode the slots and their candidate values. The maximum input length of the BERT model is set to 512. The number of heads in the slot attention module is set to 4. The output dimension of the linear transformation layer is set to 768, which is the same as the dimension of the BERT outputs. Recall that the previous turn dialogue state is treated as part of the input. The ground-truth one is used during training, and the predicted one is used during testing{\footnote{Source code is available at: \url{https://github.com/smartyfh/DST-ASSIST}}}.

We train the auxiliary model on the clean validation set and the primary model on the noisy training set.  When training the auxiliary model, the noisy training set is leveraged to choose the best model. For all primary models, the parameter $\alpha$  is set to 0.6 on MutliWOZ 2.0 and 0.4 on MultiWOZ 2.4. More training details can be found in Appendix~\ref{sec:appendixtrain}.

\section{Experimental Results}

\subsection{Main Results}

Table~\ref{tab:allmetrics} presents the performance scores of the three different primary DST models on the test sets of MultiWOZ 2.0 \& 2.4 when they are trained using our proposed framework ASSIST. For comparison, we also include the results when only the vanilla labels or only the pseudo labels are used to train the primary models. 

As can be seen, ASSIST consistently improves the performance of the three primary models on both datasets. More concretely, compared to the results obtained using only the vanilla labels, ASSIST improves the joint goal accuracy of SOM-DST, STAR, and AUX-DST on MultiWOZ 2.0 by $25.69\%$, $25.82\%$, and $28.16\%$ absolute gains, respectively. On MultiWOZ 2.4, ASSIST also leads to $8.41\%$, $5.79\%$, and $7.77\%$ absolute joint goal accuracy gains. From Table~\ref{tab:allmetrics}, we further observe that the performance improvements on MultiWOZ 2.4 are lower than on MultiWOZ 2.0. This is because the training set of MultiWOZ 2.4 is the same as that of MultiWOZ 2.1~\citep{eric-etal-2020-multiwoz}, in which lots of annotation errors have been fixed. We also observe that all the primary models demonstrate relatively good performance when only the pseudo labels are used. From these results, it can be concluded that the pseudo labels are beneficial and they can help us train DST models more robustly.

Another observation from Table~\ref{tab:allmetrics} is that SOM-DST tends to show comparable or even higher joint turn accuracy compared to STAR and AUX-DST, although its performance is worse in terms of joint goal accuracy and slot accuracy. This is because SOM-DST focuses on turn-active slots and copies the values for other slots from previous turns, while both STAR and AUX-DST predict the values of all slots from scratch at each turn. These results show that the joint turn accuracy can help us understand in more depth how different models behave.

\begin{figure}[h!]
	\centering
	\subfigure[{MultiWOZ 2.0}]{
		\includegraphics[width=0.469\columnwidth]{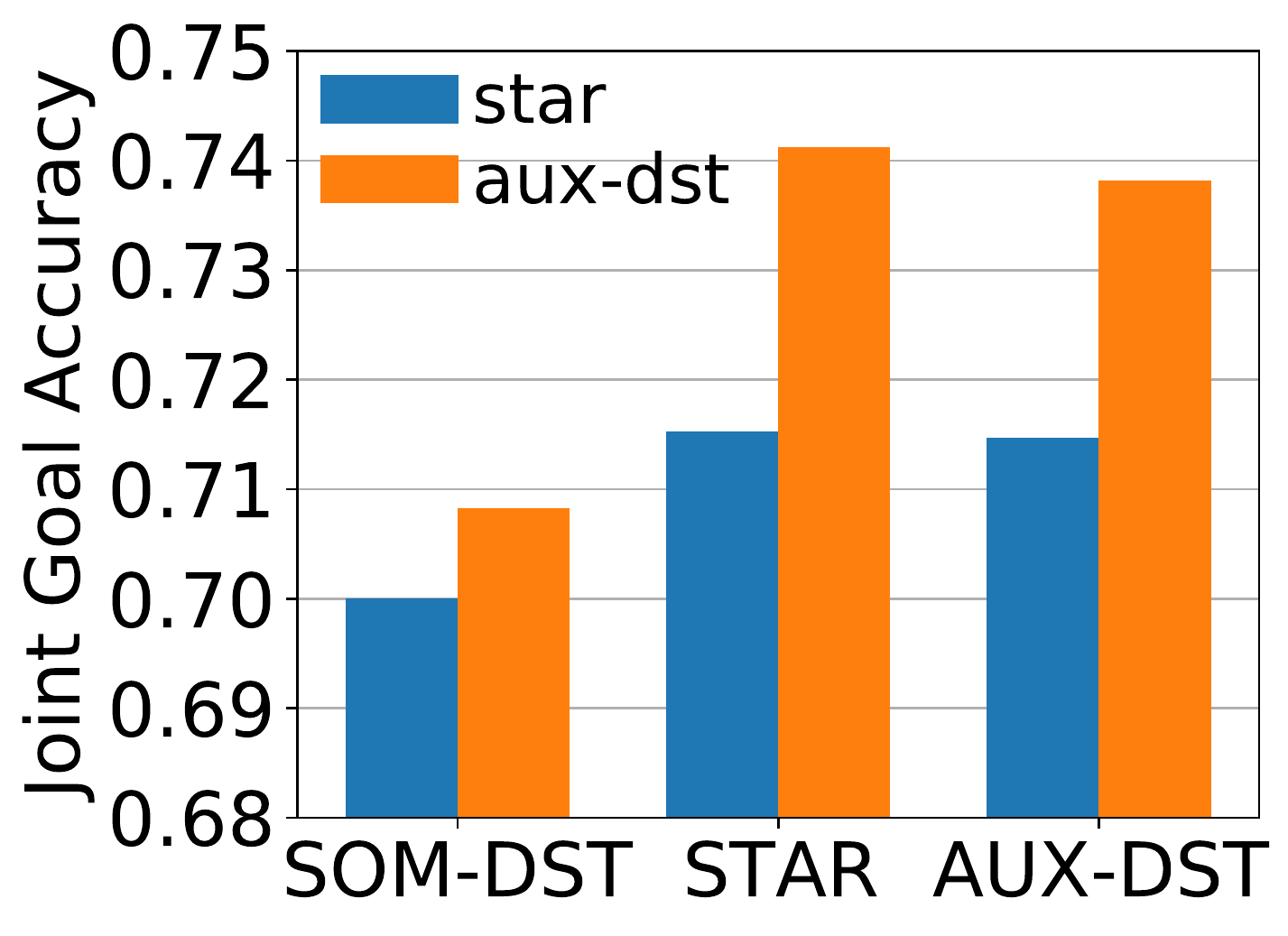}
	}
	\subfigure[{MultiWOZ 2.4}]{
		\includegraphics[width=0.469\columnwidth]{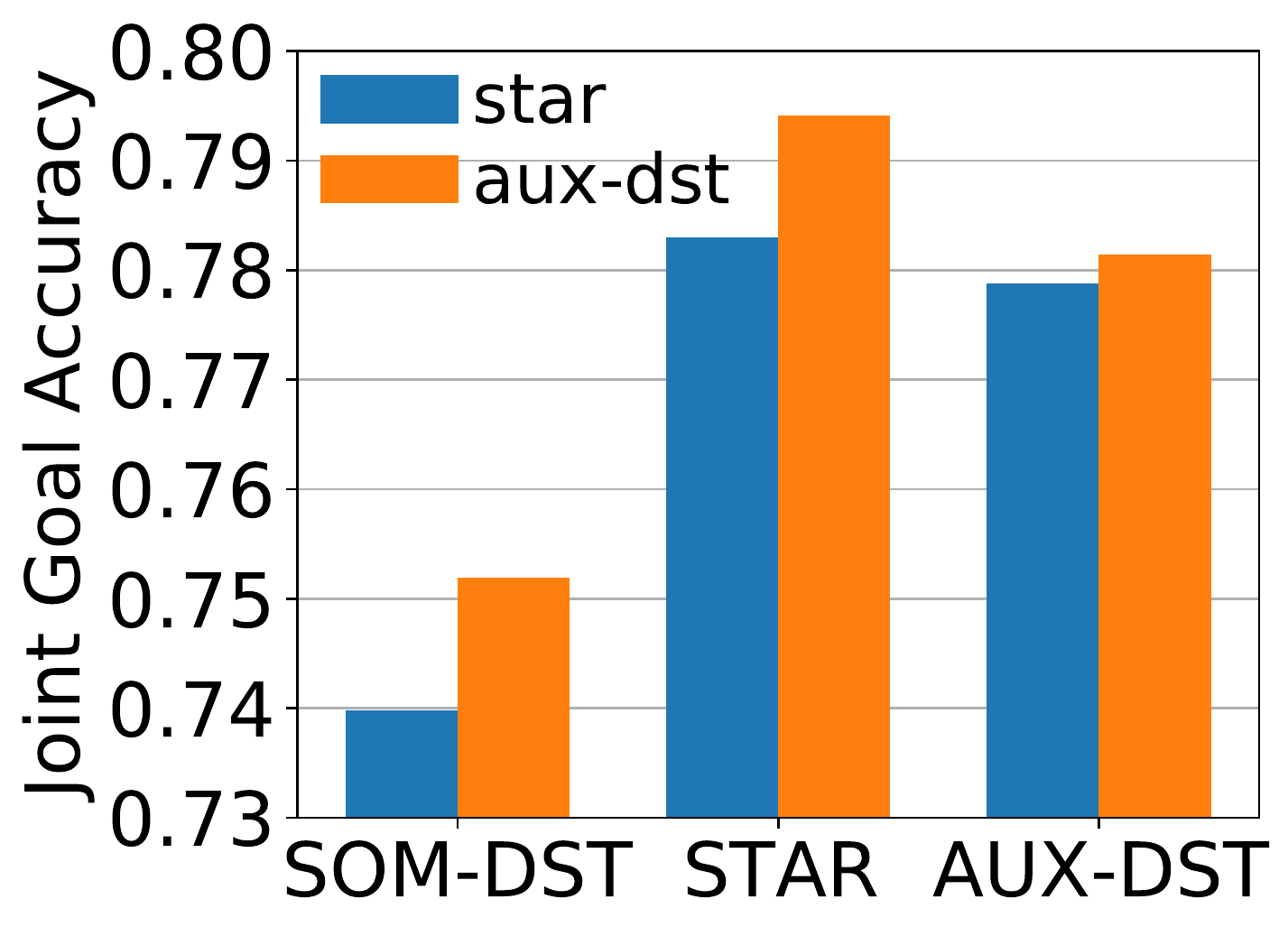}
	}
	\vspace*{-0.2cm}
	\caption{Performance comparison on MultiWOZ 2.0 and MultiWOZ 2.4 by adopting STAR as the auxiliary model. We use lowercase letters in the legend to show that the models are taken as the auxiliary model.}
	\label{fig:starsom}
\end{figure}

\subsection{Applying STAR as the Auxiliary Model}

\begin{figure*}[!t]
   \begin{minipage}{0.34\textwidth}
     \centering
     \includegraphics[width=1\linewidth]{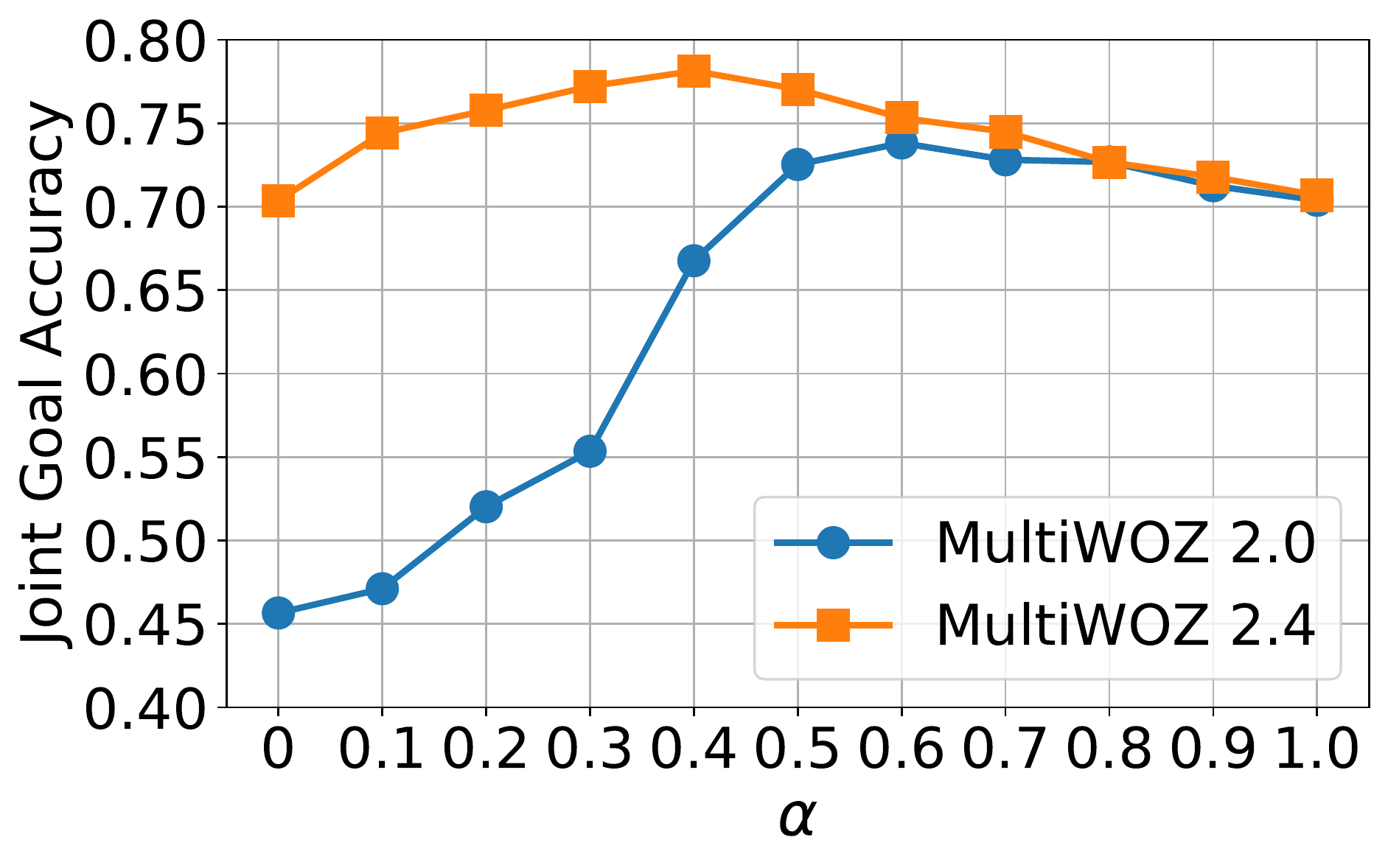}
     \caption{Effects of the parameter $\alpha$. A larger $\alpha$ indicates that more emphasis is put on the pseudo labels.}\label{fig:alpha}
   \end{minipage}\hfill
   \begin{minipage}{0.31\textwidth}
     \centering
     \includegraphics[width=0.908\linewidth]{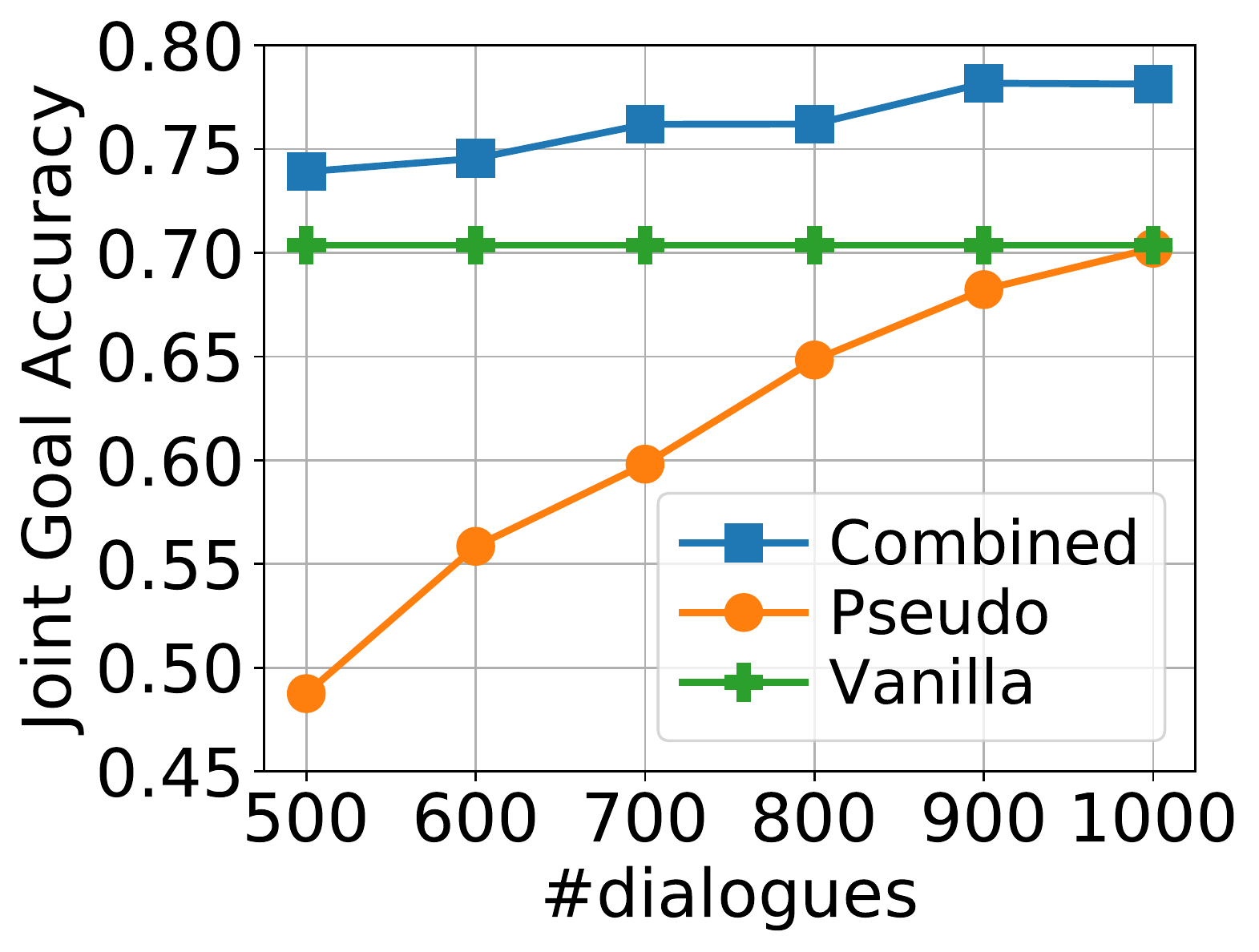}
     \caption{Effects of the size of the clean dataset. We include "Pseudo" and "Vanilla" for comparison.}\label{fig:sizeofcleanset}
   \end{minipage}\hfill
   \begin{minipage}{0.31\textwidth}
     \centering
     \includegraphics[width=0.908\linewidth]{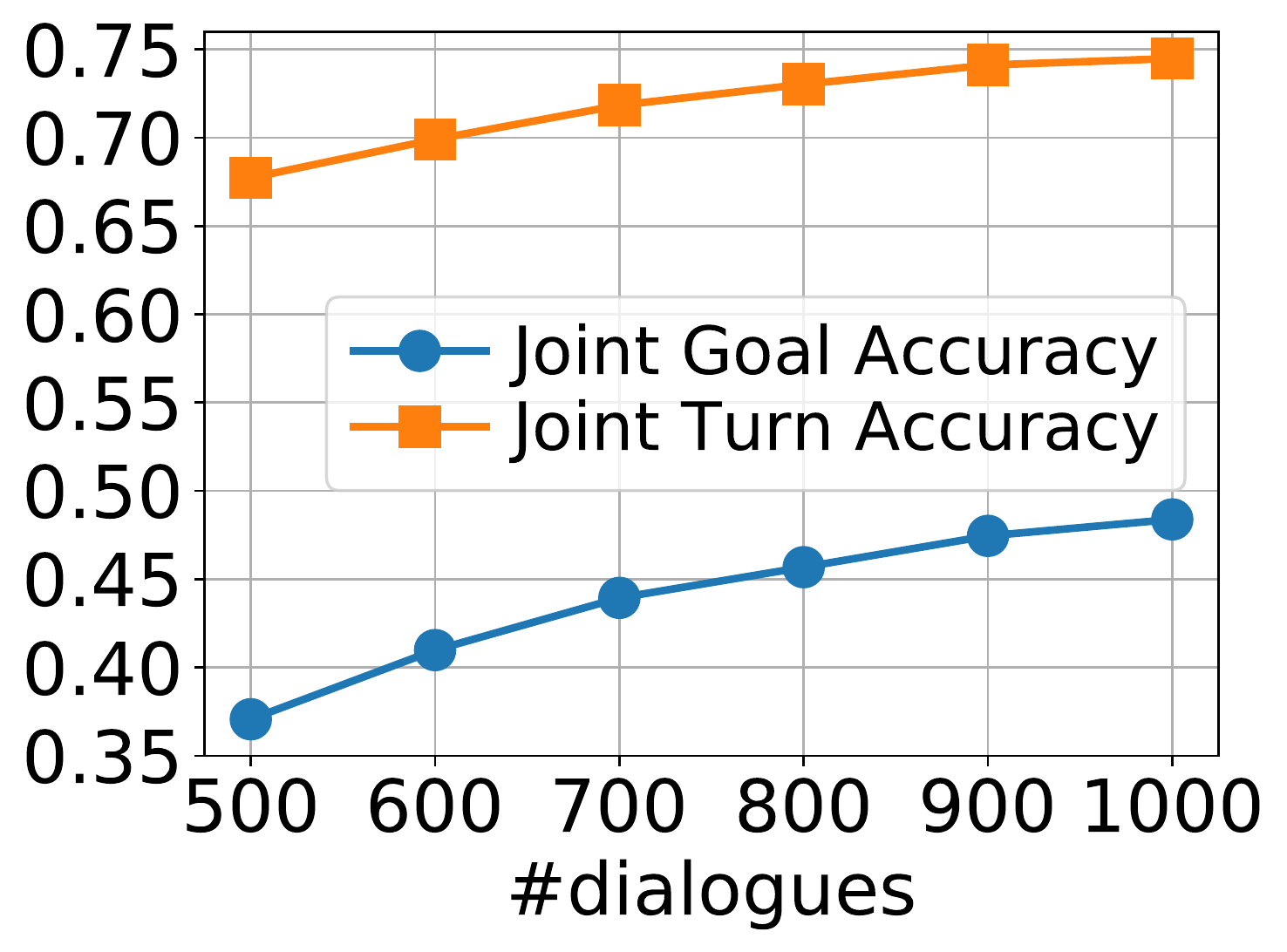}
     \caption{Performance of the auxiliary model evaluated on the noisy training set of MultiWOZ 2.4.}\label{fig:quan}
   \end{minipage}
\end{figure*}

Although any existing DST models can be adopted as the auxiliary model, we chose to propose a new simple one to reduce overfitting. In order to verify the superiority of the proposed model, we also apply STAR as the auxiliary model and compare their performance in Figure~\ref{fig:starsom}. We chose STAR due to its good performance, as shown in Table~\ref{tab:allmetrics}. From Figure~\ref{fig:starsom}, we observe that all three primary models demonstrate higher performance on both datasets when using the proposed auxiliary model than using STAR as the auxiliary model. The results indicate that the proposed auxiliary model is able to generate pseudo labels with higher quality.

\subsection{Effects of Parameter $\alpha$}


The parameter $\alpha$ adjusts the weights of the pseudo labels and vanilla labels in the training phase. Here, we study the effects of $\alpha$ by varying its value in the range of 0 to 1 with a step size of 0.1. Figure~\ref{fig:alpha} shows the results of AUX-DST. As can be seen, $\alpha$ plays an important role in balancing the pseudo labels and vanilla labels. The best performance is achieved when $\alpha$ is set to 0.6 on MultiWOZ 2.0 and 0.4 on MultiWOZ 2.4. Since the training set of MultiWOZ 2.0 has more noisy labels than that of MultiWOZ 2.4, more emphasis should be put on its pseudo labels to obtain the best performance. It is also noted that the performance difference between MultiWOZ 2.0 and MultiWOZ 2.4 dwindles away as $\alpha$ increases. This is because the vanilla labels will contribute less to the training of the primary model when $\alpha$ is set to be larger.

\subsection{Effects of the Size of the Clean Dataset}

Considering that our proposed framework ASSIST relies on a small clean dataset to train the auxiliary model that is further leveraged to generate pseudo labels for the training set, it is valuable to explore the effects of the size of the clean dataset on the performance of the primary model. For this purpose, we vary the number of dialogues in the clean dataset from 500 to 1000\footnote{There are 1000 dialogues in total in the validation set.} to generate different pseudo labels. We then combine these different pseudo labels with the vanilla labels to train the primary model AUX-DST. The results on MultiWOZ 2.4 are reported in Figure~\ref{fig:sizeofcleanset}. For comparison, we also include the results when only the pseudo labels or only the vanilla labels are used to train the primary model. As can be seen, the size of the clean dataset has a great impact on the performance of the primary model. Apparently, fewer clean data will lead to worse performance. Nevertheless, as long as the pseudo labels are combined with the vanilla labels, the primary model can consistently demonstrate the strongest performance.

\subsection{Analyses on Pseudo Labels' Quality}

The previous experiments have proven the effectiveness of the generated pseudo labels in training robust DST models. In this part, we provide further analyses on the quality of the pseudo labels to gain more insights into why they can be beneficial.

\subsubsection{Quantitative Analysis}


We first investigate whether the pseudo labels are consistent with the true labels. To achieve this goal, we can compute the joint goal accuracy and joint turn accuracy of the auxiliary model on the training set. However, the true labels of the training set are unavailable. As an alternative, we treat the vanilla noisy labels as true labels (note that only a portion of the vanilla labels are noisy). In this experiment, we also vary the number of clean dialogues to train the auxiliary model. Figure~\ref{fig:quan} presents the results. As shown in Figure~\ref{fig:quan}, the auxiliary model achieves higher performance when more clean dialogues are utilized to train it. If the entire validation set is used, it achieves around $50\%$ joint goal accuracy and around $75\%$ joint turn accuracy. Given that the vanilla noisy labels are regarded as the true labels, we can conjecture that the true performance is actually higher. This experiment shows that the pseudo labels are consistent with the unknown true labels to some extent and can serve as a good complement to the vanilla noisy labels.


\colorlet{Mycolor1}{green!5!orange!95!}
\definecolor{Mycolor2}{HTML}{5e1bea}
\definecolor{Mycolor3}{HTML}{ec198f} 
\begin{table*}[t!]
\small
\centering
\setlength{\tabcolsep}{1.0mm}
\begin{tabular}{l|c|c}
\hline
\multicolumn{1}{c|}{\textbf{Dialogue Context}}                                                                                                                                                              & \textbf{Vanilla Labels}                                                                                                                   & \textbf{Pseudo Labels}                                                                                                                                  \\ \hline \hline

\begin{tabular}[c]{@{}l@{}}\textbf{{[}sys{]}:} Sure,  \textcolor{Mycolor1}{da vinci pizzeria} is a cheap Italian \\ restaurant in the area.\\ \textbf{{[}usr{]}:} Would you mind making a reservation for\\  \textcolor{Mycolor1}{Thursday} at  \textcolor{Mycolor1}{17:15}?\end{tabular} & (restaurant-name, da vinci pizzeria)                                                                                             & \begin{tabular}[c]{@{}c@{}}\textcolor{Mycolor2}{(restaurant-book day, thursday)}\\ \textcolor{Mycolor2}{(restaurant-book time, 17:15)}\\ (restaurant-name, da vinci pizzeria)\end{tabular} \\ \hline
\begin{tabular}[c]{@{}l@{}}\textbf{{[}sys{]}:} Do you have a preferred section of town?\\ \textbf{{[}usr{]}:} \textcolor{Mycolor1}{Not really}, but I want  \textcolor{Mycolor1}{free wifi} and it\\ should be  \textcolor{Mycolor1}{4 star}.\end{tabular}                                 & \begin{tabular}[c]{@{}c@{}}(hotel-internet, free)\\ (hotel-stars, 4)\end{tabular}                                                & \begin{tabular}[c]{@{}c@{}} \textcolor{Mycolor2}{(hotel-area, dontcare)}\\ (hotel-internet, free)\\ (hotel-stars, 4)\end{tabular}                                     \\ \hline
\begin{tabular}[c]{@{}l@{}}\textbf{{[}usr{]}:} I need to find out if there is a train going\\ to  \textcolor{Mycolor1}{stansted airport} that leaves after  \textcolor{Mycolor1}{12:30}.\end{tabular}                                                     & \begin{tabular}[c]{@{}c@{}}\textcolor{Mycolor3}{(train-arriveby, 13:03)}\\ (train-destination, stansted airport)\\ (train-leaveat, 12:30)\end{tabular} & \begin{tabular}[c]{@{}c@{}}(train-destination, stansted airport)\\ (train-leaveat, 12:30)\end{tabular}                                         \\ \hline
\begin{tabular}[c]{@{}l@{}}\textbf{{[}usr{]}:} I am staying in the  \textcolor{Mycolor1}{west} part of Cambridge\\ and would like to know about some  \textcolor{Mycolor1}{places to go}.\end{tabular}                                                    & (attraction-area, west)                                                                                                          & \begin{tabular}[c]{@{}c@{}}(attraction-area, west)\\ \textcolor{Mycolor3}{(hotel-area, west)}\end{tabular}                                                           \\ \hline
\end{tabular}
\caption{Four dialogue snippets with their vanilla labels and the generated pseudo labels. These dialogue snippets are chosen from the training set of MultiWOZ 2.4. To save space, we only present turn-active slots and their values. }
\label{tab:case}
\end{table*}

\subsubsection{Qualitative Analysis}

To intuitively understand the quality of the pseudo labels, we show four dialogue snippets with their vanilla labels and the generated pseudo labels in Table~\ref{tab:case}. As can be seen, the vanilla labels of the first two dialogue snippets are incomplete, while all the missing information is presented in the pseudo labels. For the third dialogue snippet, the vanilla labels contain an unmentioned slot-value pair "\textit{(train-arriveby, 13:03)}". This error has also been fixed in the pseudo labels. For the last dialogue snippet, the vanilla labels are correct. However, the pseudo labels introduce an overconfident prediction of the value of slot "\textit{hotel-area}". This case study has verified again that the pseudo labels can be utilized to fix certain errors in the vanilla labels. However, the pseudo labels may bring about some new errors. Hence, we should combine the two types of labels so as to achieve the best performance.

\begin{table}[h]
\begin{tabular}{l|c|c}
\hline
 \textbf{Settings} & \textbf{MultiWOZ 2.0} & \textbf{MultiWOZ 2.4} \\ \hline
T & 45.66 & 71.80 \\
T+C & 50.75 & 76.89 \\
T+P & 73.82 & 78.47 \\ 
T+C+P & 74.96 & 78.92 \\ \hline
\end{tabular}
\caption{The joint goal accuracy ($\%$) of AUX-DST on MultiWOZ 2.0 \& 2.4 under different training settings. T: the noisy training set. C: the small clean dataset. P: the generated pseudo labels of the original training set. The reported scores are the best ones on the test set.}
\label{tab:combinedaaa}
\end{table}

\subsection{Pseudo Labels vs. Simple Combination}

Aiming to better validate the effectiveness of the proposed framework, we also report the results when the small clean dataset is directly combined with the large noisy training set to train the primary model. We adopt AUX-DST as the primary model and show the results in Table~\ref{tab:combinedaaa}. Since the clean dataset (i.e., the validation set in our experiments) is combined with the training set, all the results in Table~\ref{tab:combinedaaa} are the best ones on the test set. As can be observed, a simple combination of the noisy training set and clean dataset can lead to better results. However, the performance improvements are lower, compared to using pseudo labels (especially on MultiWOZ 2.0 due to its noisier training set). It is also observed that when both the clean dataset and the pseudo labels are utilized to train the model, even higher performance can be achieved. These results indicate that our proposed framework can make better use of the small clean dataset to train the primary model.

\subsection{Error Analysis}

\begin{figure}[t!]
  \centering
  \includegraphics[width=1.0\linewidth]{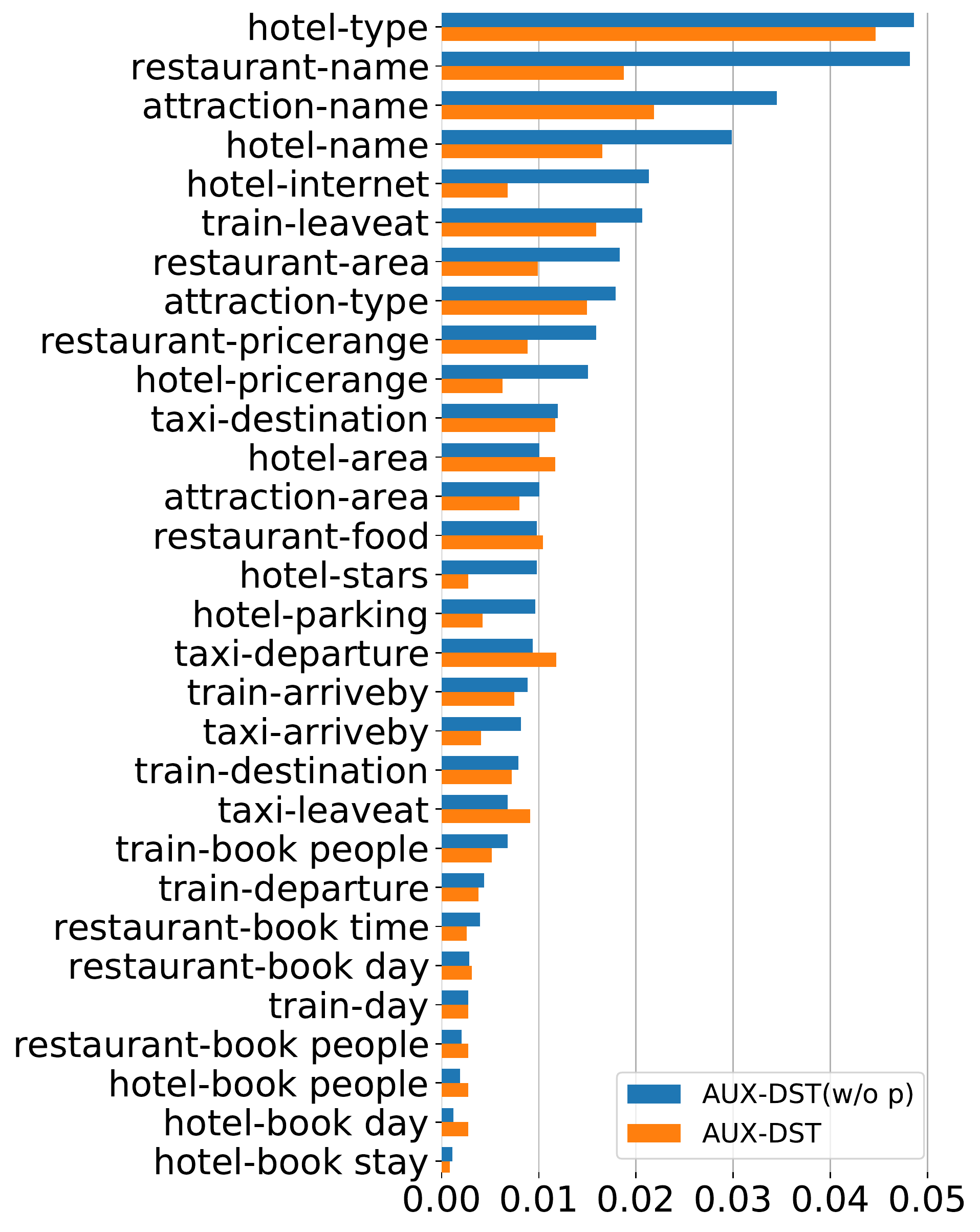}
  \caption{The error rate of each slot on MultiWOZ 2.4.}
  \label{fig:err}
\end{figure}

We further investigate the error rate with respect to each slot. We adopt AUX-DST as the primary model and use AUX-DST(w/o p) to denote the case when only the vanilla labels are employed to train the model. The results on the test set of MultiWOZ 2.4 are illustrated in Figure~\ref{fig:err}, from which we can observe that the slot "\textit{hotel-type}" has the highest error rate. Even though the error rate is reduced with the aid of the pseudo labels, it is still the highest one among all the slots. This is because the labels of this slot are confusing. It is also observed that the "\textit{name}"-related slots have relatively high error rates. However, when the pseudo labels are used, their error rates reduce remarkably. Besides, we observe that the error rates of some slots are higher when the pseudo labels are leveraged. This is probably due to the fact that we have used the same parameter $\alpha$ to combine the pseudo labels and vanilla labels of all slots. In practice, the noise rate with respect to each slot in the vanilla labels may not be exactly the same. This observation inspires us that more advanced techniques should be developed to combine the pseudo labels and vanilla labels, which we leave as our future work.

\section{Related Work}

In this section, we briefly review related work on DST and noisy label learning.

\subsection{Dialogue State Tracking}

Recently, DST has got an enormous amount of attention, thanks to the availability of multiple large-scale multi-domain dialogue datasets such as MultiWOZ 2.0~\citep{budzianowski-etal-2018-multiwoz}, MultiWOZ 2.1~\citep{eric-etal-2020-multiwoz}, RiSAWOZ~\citep{quan-etal-2020-risawoz}, and SGD~\citep{rastogi2020towards}. The most popular datasets are MultiWOZ 2.0 and MultiWOZ 2.1, and lots of DST models have been built on top of them~\citep{lee-etal-2019-sumbt, wu-etal-2019-transferable, ouyang-etal-2020-dialogue, hosseini2020simple, kim-etal-2020-efficient, hu-etal-2020-sas, feng2020sequence, ye2021slot, lin2021knowledge, liang2021attention}.

These recent DST models can be grouped into two categories: predefined ontology-based models and open vocabulary-based models. The predefined ontology-based models treat DST as a multi-label classification problem and tend to demonstrate better performance~\citep{chen2020schema, zhang-etal-2020-find, shan-etal-2020-contextual, ye2021slot}. The open vocabulary-based models leverage either span prediction~\citep{heck-etal-2020-trippy, gao-etal-2020-machine} or sequence generation~\citep{wu-etal-2019-transferable, feng2020sequence, hosseini2020simple} to extract slot values from the dialogue context directly.

Although these DST models have made a huge success, they can only achieve sub-optimal performance, due to the lack of handling noisy labels. To the best of our knowledge, we are the first to take the noisy labels into consideration when tackling the DST problem.

\subsection{Noisy Label Learning}

Addressing noisy labels in supervised learning is a long-term studied problem~\citep{frenay2013classification, song2020learning, han2020survey}. This issue becomes more prominent in the era of deep learning, as training deep models generally requires a lot of well-labelled data, but it is expensive and time-consuming to collect large-scale datasets with completely clean annotations. This dilemma has sparked a surge of noisy label learning methods~\citep{hendrycks2018using, zhang2018generalized, song2019selfie, wei2020combating}. Even so, these methods mainly focus on multi-class classification~\citep{song2020learning}, which makes it not straightforward to apply them to the DST task.

\section{Conclusion}

In this work, we have presented a general framework ASSIST, aiming to train DST models robustly from noisy labels. ASSIST leverages an auxiliary model that is trained on a small clean dataset to generate pseudo labels for the large noisy training set. The pseudo labels are combined with the vanilla labels to train the primary model. Both theoretical analysis and empirical study have verified the validity of our proposed framework. In the future, we intend to explore more advanced techniques to combine the pseudo labels and vanilla noisy labels in a better way.

\section*{Acknowledgments}
This project was funded by the EPSRC Fellowship titled “Task Based Information Retrieval” and grant reference number EP/P024289/1.

\bibliography{anthology,custom}

\begin{thebibliography}{45}
\expandafter\ifx\csname natexlab\endcsname\relax\def\natexlab#1{#1}\fi

\bibitem[{Ba et~al.(2016)Ba, Kiros, and Hinton}]{ba2016layer}
Jimmy~Lei Ba, Jamie~Ryan Kiros, and Geoffrey~E Hinton. 2016.
\newblock Layer normalization.
\newblock \emph{arXiv preprint arXiv:1607.06450}.

\bibitem[{Bowman et~al.(2016)Bowman, Vilnis, Vinyals, Dai, Jozefowicz, and
  Bengio}]{bowman-etal-2016-generating}
Samuel~R. Bowman, Luke Vilnis, Oriol Vinyals, Andrew Dai, Rafal Jozefowicz, and
  Samy Bengio. 2016.
\newblock \href {https://doi.org/10.18653/v1/K16-1002} {Generating sentences
  from a continuous space}.
\newblock In \emph{Proceedings of The 20th {SIGNLL} Conference on Computational
  Natural Language Learning}, pages 10--21, Berlin, Germany. Association for
  Computational Linguistics.

\bibitem[{Budzianowski et~al.(2018)Budzianowski, Wen, Tseng, Casanueva, Ultes,
  Ramadan, and Ga{\v{s}}i{\'c}}]{budzianowski-etal-2018-multiwoz}
Pawe{\l} Budzianowski, Tsung-Hsien Wen, Bo-Hsiang Tseng, I{\~n}igo Casanueva,
  Stefan Ultes, Osman Ramadan, and Milica Ga{\v{s}}i{\'c}. 2018.
\newblock \href {https://doi.org/10.18653/v1/D18-1547} {{M}ulti{WOZ} - a
  large-scale multi-domain {W}izard-of-{O}z dataset for task-oriented dialogue
  modelling}.
\newblock In \emph{Proceedings of the 2018 Conference on Empirical Methods in
  Natural Language Processing}, pages 5016--5026, Brussels, Belgium.
  Association for Computational Linguistics.

\bibitem[{Chen et~al.(2020)Chen, Lv, Wang, Zhu, Tan, and Yu}]{chen2020schema}
Lu~Chen, Boer Lv, Chi Wang, Su~Zhu, Bowen Tan, and Kai Yu. 2020.
\newblock Schema-guided multi-domain dialogue state tracking with graph
  attention neural networks.
\newblock In \emph{Proceedings of the AAAI Conference on Artificial
  Intelligence}, volume~34, pages 7521--7528.

\bibitem[{Chen et~al.(2017)Chen, Celikyilmaz, and
  Hakkani-T{\"u}r}]{chen-etal-2017-deep}
Yun-Nung Chen, Asli Celikyilmaz, and Dilek Hakkani-T{\"u}r. 2017.
\newblock \href {https://www.aclweb.org/anthology/P17-5004} {Deep learning for
  dialogue systems}.
\newblock In \emph{Proceedings of the 55th Annual Meeting of the Association
  for Computational Linguistics: Tutorial Abstracts}, pages 8--14, Vancouver,
  Canada. Association for Computational Linguistics.

\bibitem[{Devlin et~al.(2019)Devlin, Chang, Lee, and
  Toutanova}]{devlin-etal-2019-bert}
Jacob Devlin, Ming-Wei Chang, Kenton Lee, and Kristina Toutanova. 2019.
\newblock \href {https://doi.org/10.18653/v1/N19-1423} {{BERT}: Pre-training of
  deep bidirectional transformers for language understanding}.
\newblock In \emph{Proceedings of the 2019 Conference of the North {A}merican
  Chapter of the Association for Computational Linguistics: Human Language
  Technologies, Volume 1 (Long and Short Papers)}, pages 4171--4186,
  Minneapolis, Minnesota. Association for Computational Linguistics.

\bibitem[{El~Asri et~al.(2017)El~Asri, Schulz, Sharma, Zumer, Harris, Fine,
  Mehrotra, and Suleman}]{el-asri-etal-2017-frames}
Layla El~Asri, Hannes Schulz, Shikhar Sharma, Jeremie Zumer, Justin Harris,
  Emery Fine, Rahul Mehrotra, and Kaheer Suleman. 2017.
\newblock \href {https://doi.org/10.18653/v1/W17-5526} {{F}rames: a corpus for
  adding memory to goal-oriented dialogue systems}.
\newblock In \emph{Proceedings of the 18th Annual {SIG}dial Meeting on
  Discourse and Dialogue}, pages 207--219, Saarbr{\"u}cken, Germany.
  Association for Computational Linguistics.

\bibitem[{Eric et~al.(2020)Eric, Goel, Paul, Sethi, Agarwal, Gao, Kumar, Goyal,
  Ku, and Hakkani-Tur}]{eric-etal-2020-multiwoz}
Mihail Eric, Rahul Goel, Shachi Paul, Abhishek Sethi, Sanchit Agarwal, Shuyang
  Gao, Adarsh Kumar, Anuj Goyal, Peter Ku, and Dilek Hakkani-Tur. 2020.
\newblock \href {https://www.aclweb.org/anthology/2020.lrec-1.53} {{M}ulti{WOZ}
  2.1: A consolidated multi-domain dialogue dataset with state corrections and
  state tracking baselines}.
\newblock In \emph{Proceedings of the 12th Language Resources and Evaluation
  Conference}, pages 422--428, Marseille, France. European Language Resources
  Association.

\bibitem[{Feng et~al.(2020)Feng, Wang, and Li}]{feng2020sequence}
Yue Feng, Yang Wang, and Hang Li. 2020.
\newblock A sequence-to-sequence approach to dialogue state tracking.
\newblock \emph{arXiv preprint arXiv:2011.09553}.

\bibitem[{Fr{\'e}nay and Verleysen(2013)}]{frenay2013classification}
Beno{\^\i}t Fr{\'e}nay and Michel Verleysen. 2013.
\newblock Classification in the presence of label noise: a survey.
\newblock \emph{IEEE transactions on neural networks and learning systems},
  25(5):845--869.

\bibitem[{Gao et~al.(2019)Gao, Galley, Li et~al.}]{gao2019neural}
Jianfeng Gao, Michel Galley, Lihong Li, et~al. 2019.
\newblock Neural approaches to conversational ai.
\newblock \emph{Foundations and Trends{\textregistered} in Information
  Retrieval}, 13(2-3):127--298.

\bibitem[{Gao et~al.(2020)Gao, Agarwal, Jin, Chung, and
  Hakkani-Tur}]{gao-etal-2020-machine}
Shuyang Gao, Sanchit Agarwal, Di~Jin, Tagyoung Chung, and Dilek Hakkani-Tur.
  2020.
\newblock \href {https://doi.org/10.18653/v1/2020.nlp4convai-1.10} {From
  machine reading comprehension to dialogue state tracking: Bridging the gap}.
\newblock In \emph{Proceedings of the 2nd Workshop on Natural Language
  Processing for Conversational AI}, pages 79--89, Online. Association for
  Computational Linguistics.

\bibitem[{Han et~al.(2020{\natexlab{a}})Han, Yao, Liu, Niu, Tsang, Kwok, and
  Sugiyama}]{han2020survey}
Bo~Han, Quanming Yao, Tongliang Liu, Gang Niu, Ivor~W Tsang, James~T Kwok, and
  Masashi Sugiyama. 2020{\natexlab{a}}.
\newblock A survey of label-noise representation learning: Past, present and
  future.
\newblock \emph{arXiv preprint arXiv:2011.04406}.

\bibitem[{Han et~al.(2020{\natexlab{b}})Han, Liu, Takanobu, Lian, Huang, Peng,
  and Huang}]{han2020multiwoz}
Ting Han, Ximing Liu, Ryuichi Takanobu, Yixin Lian, Chongxuan Huang, Wei Peng,
  and Minlie Huang. 2020{\natexlab{b}}.
\newblock Multiwoz 2.3: A multi-domain task-oriented dataset enhanced with
  annotation corrections and co-reference annotation.
\newblock \emph{arXiv preprint arXiv:2010.05594}.

\bibitem[{Heck et~al.(2020)Heck, van Niekerk, Lubis, Geishauser, Lin, Moresi,
  and Gasic}]{heck-etal-2020-trippy}
Michael Heck, Carel van Niekerk, Nurul Lubis, Christian Geishauser, Hsien-Chin
  Lin, Marco Moresi, and Milica Gasic. 2020.
\newblock \href {https://www.aclweb.org/anthology/2020.sigdial-1.4}
  {{T}rip{P}y: A triple copy strategy for value independent neural dialog state
  tracking}.
\newblock In \emph{Proceedings of the 21th Annual Meeting of the Special
  Interest Group on Discourse and Dialogue}, pages 35--44, 1st virtual meeting.
  Association for Computational Linguistics.

\bibitem[{Henderson et~al.(2014)Henderson, Thomson, and
  Williams}]{henderson-etal-2014-second}
Matthew Henderson, Blaise Thomson, and Jason~D. Williams. 2014.
\newblock \href {https://doi.org/10.3115/v1/W14-4337} {The second dialog state
  tracking challenge}.
\newblock In \emph{Proceedings of the 15th Annual Meeting of the Special
  Interest Group on Discourse and Dialogue ({SIGDIAL})}, pages 263--272,
  Philadelphia, PA, U.S.A. Association for Computational Linguistics.

\bibitem[{Hendrycks et~al.(2018)Hendrycks, Mazeika, Wilson, and
  Gimpel}]{hendrycks2018using}
Dan Hendrycks, Mantas Mazeika, Duncan Wilson, and Kevin Gimpel. 2018.
\newblock Using trusted data to train deep networks on labels corrupted by
  severe noise.
\newblock In \emph{Proceedings of the 32nd International Conference on Neural
  Information Processing Systems}, pages 10477--10486.

\bibitem[{Hosseini-Asl et~al.(2020)Hosseini-Asl, McCann, Wu, Yavuz, and
  Socher}]{hosseini2020simple}
Ehsan Hosseini-Asl, Bryan McCann, Chien-Sheng Wu, Semih Yavuz, and Richard
  Socher. 2020.
\newblock A simple language model for task-oriented dialogue.
\newblock \emph{arXiv preprint arXiv:2005.00796}.

\bibitem[{Hu et~al.(2020)Hu, Yang, Chen, He, and Yu}]{hu-etal-2020-sas}
Jiaying Hu, Yan Yang, Chencai Chen, Liang He, and Zhou Yu. 2020.
\newblock \href {https://doi.org/10.18653/v1/2020.acl-main.567} {{SAS}:
  Dialogue state tracking via slot attention and slot information sharing}.
\newblock In \emph{Proceedings of the 58th Annual Meeting of the Association
  for Computational Linguistics}, pages 6366--6375, Online. Association for
  Computational Linguistics.

\bibitem[{Kim et~al.(2020)Kim, Yang, Kim, and Lee}]{kim-etal-2020-efficient}
Sungdong Kim, Sohee Yang, Gyuwan Kim, and Sang-Woo Lee. 2020.
\newblock \href {https://doi.org/10.18653/v1/2020.acl-main.53} {Efficient
  dialogue state tracking by selectively overwriting memory}.
\newblock In \emph{Proceedings of the 58th Annual Meeting of the Association
  for Computational Linguistics}, pages 567--582, Online. Association for
  Computational Linguistics.

\bibitem[{Kingma and Ba(2014)}]{kingma2014adam}
Diederik~P Kingma and Jimmy Ba. 2014.
\newblock Adam: A method for stochastic optimization.
\newblock \emph{arXiv preprint arXiv:1412.6980}.

\bibitem[{Lee et~al.(2019)Lee, Lee, and Kim}]{lee-etal-2019-sumbt}
Hwaran Lee, Jinsik Lee, and Tae-Yoon Kim. 2019.
\newblock \href {https://doi.org/10.18653/v1/P19-1546} {{SUMBT}: Slot-utterance
  matching for universal and scalable belief tracking}.
\newblock In \emph{Proceedings of the 57th Annual Meeting of the Association
  for Computational Linguistics}, pages 5478--5483, Florence, Italy.
  Association for Computational Linguistics.

\bibitem[{Liang et~al.(2021)Liang, Poddar, and Szarvas}]{liang2021attention}
Shuailong Liang, Lahari Poddar, and Gyuri Szarvas. 2021.
\newblock Attention guided dialogue state tracking with sparse supervision.
\newblock \emph{arXiv preprint arXiv:2101.11958}.

\bibitem[{Lin et~al.(2021)Lin, Tseng, and Byrne}]{lin2021knowledge}
Weizhe Lin, Bo-Hsian Tseng, and Bill Byrne. 2021.
\newblock Knowledge-aware graph-enhanced gpt-2 for dialogue state tracking.
\newblock \emph{arXiv preprint arXiv:2104.04466}.

\bibitem[{Mrk{\v{s}}i{\'c} et~al.(2017)Mrk{\v{s}}i{\'c}, {\'O}~S{\'e}aghdha,
  Wen, Thomson, and Young}]{mrksic-etal-2017-neural}
Nikola Mrk{\v{s}}i{\'c}, Diarmuid {\'O}~S{\'e}aghdha, Tsung-Hsien Wen, Blaise
  Thomson, and Steve Young. 2017.
\newblock \href {https://doi.org/10.18653/v1/P17-1163} {Neural belief tracker:
  Data-driven dialogue state tracking}.
\newblock In \emph{Proceedings of the 55th Annual Meeting of the Association
  for Computational Linguistics (Volume 1: Long Papers)}, pages 1777--1788,
  Vancouver, Canada. Association for Computational Linguistics.

\bibitem[{Natarajan et~al.(2013)Natarajan, Dhillon, Ravikumar, and
  Tewari}]{natarajan2013learning}
Nagarajan Natarajan, Inderjit~S Dhillon, Pradeep Ravikumar, and Ambuj Tewari.
  2013.
\newblock Learning with noisy labels.
\newblock In \emph{NIPS}, volume~26, pages 1196--1204.

\bibitem[{Ouyang et~al.(2020)Ouyang, Chen, Dai, Zhao, Huang, and
  Chen}]{ouyang-etal-2020-dialogue}
Yawen Ouyang, Moxin Chen, Xinyu Dai, Yinggong Zhao, Shujian Huang, and Jiajun
  Chen. 2020.
\newblock \href {https://doi.org/10.18653/v1/2020.acl-main.5} {Dialogue state
  tracking with explicit slot connection modeling}.
\newblock In \emph{Proceedings of the 58th Annual Meeting of the Association
  for Computational Linguistics}, pages 34--40, Online. Association for
  Computational Linguistics.

\bibitem[{Quan et~al.(2020)Quan, Zhang, Cao, Li, and
  Xiong}]{quan-etal-2020-risawoz}
Jun Quan, Shian Zhang, Qian Cao, Zizhong Li, and Deyi Xiong. 2020.
\newblock \href {https://doi.org/10.18653/v1/2020.emnlp-main.67} {{R}i{SAWOZ}:
  A large-scale multi-domain {W}izard-of-{O}z dataset with rich semantic
  annotations for task-oriented dialogue modeling}.
\newblock In \emph{Proceedings of the 2020 Conference on Empirical Methods in
  Natural Language Processing (EMNLP)}, pages 930--940, Online. Association for
  Computational Linguistics.

\bibitem[{Rastogi et~al.(2020)Rastogi, Zang, Sunkara, Gupta, and
  Khaitan}]{rastogi2020towards}
Abhinav Rastogi, Xiaoxue Zang, Srinivas Sunkara, Raghav Gupta, and Pranav
  Khaitan. 2020.
\newblock Towards scalable multi-domain conversational agents: The
  schema-guided dialogue dataset.
\newblock In \emph{Proceedings of the AAAI Conference on Artificial
  Intelligence}, volume~34, pages 8689--8696.

\bibitem[{Shan et~al.(2020)Shan, Li, Zhang, Meng, Feng, Niu, and
  Zhou}]{shan-etal-2020-contextual}
Yong Shan, Zekang Li, Jinchao Zhang, Fandong Meng, Yang Feng, Cheng Niu, and
  Jie Zhou. 2020.
\newblock \href {https://doi.org/10.18653/v1/2020.acl-main.563} {A contextual
  hierarchical attention network with adaptive objective for dialogue state
  tracking}.
\newblock In \emph{Proceedings of the 58th Annual Meeting of the Association
  for Computational Linguistics}, pages 6322--6333, Online. Association for
  Computational Linguistics.

\bibitem[{Song et~al.(2019)Song, Kim, and Lee}]{song2019selfie}
Hwanjun Song, Minseok Kim, and Jae-Gil Lee. 2019.
\newblock Selfie: Refurbishing unclean samples for robust deep learning.
\newblock In \emph{International Conference on Machine Learning}, pages
  5907--5915. PMLR.

\bibitem[{Song et~al.(2020)Song, Kim, Park, Shin, and Lee}]{song2020learning}
Hwanjun Song, Minseok Kim, Dongmin Park, Yooju Shin, and Jae-Gil Lee. 2020.
\newblock Learning from noisy labels with deep neural networks: A survey.
\newblock \emph{arXiv preprint arXiv:2007.08199}.

\bibitem[{Srivastava et~al.(2014)Srivastava, Hinton, Krizhevsky, Sutskever, and
  Salakhutdinov}]{srivastava2014dropout}
Nitish Srivastava, Geoffrey Hinton, Alex Krizhevsky, Ilya Sutskever, and Ruslan
  Salakhutdinov. 2014.
\newblock Dropout: a simple way to prevent neural networks from overfitting.
\newblock \emph{The journal of machine learning research}, 15(1):1929--1958.

\bibitem[{Vaswani et~al.(2017)Vaswani, Shazeer, Parmar, Uszkoreit, Jones,
  Gomez, Kaiser, and Polosukhin}]{vaswani2017attention}
Ashish Vaswani, Noam Shazeer, Niki Parmar, Jakob Uszkoreit, Llion Jones,
  Aidan~N Gomez, Lukasz Kaiser, and Illia Polosukhin. 2017.
\newblock Attention is all you need.
\newblock \emph{arXiv preprint arXiv:1706.03762}.

\bibitem[{Wei et~al.(2020)Wei, Feng, Chen, and An}]{wei2020combating}
Hongxin Wei, Lei Feng, Xiangyu Chen, and Bo~An. 2020.
\newblock Combating noisy labels by agreement: A joint training method with
  co-regularization.
\newblock In \emph{Proceedings of the IEEE/CVF Conference on Computer Vision
  and Pattern Recognition}, pages 13726--13735.

\bibitem[{Williams et~al.(2014)Williams, Henderson, Raux, Thomson, Black, and
  Ramachandran}]{williams2014dialog}
Jason~D Williams, Matthew Henderson, Antoine Raux, Blaise Thomson, Alan Black,
  and Deepak Ramachandran. 2014.
\newblock The dialog state tracking challenge series.
\newblock \emph{AI Magazine}, 35(4):121--124.

\bibitem[{Wu et~al.(2019)Wu, Madotto, Hosseini-Asl, Xiong, Socher, and
  Fung}]{wu-etal-2019-transferable}
Chien-Sheng Wu, Andrea Madotto, Ehsan Hosseini-Asl, Caiming Xiong, Richard
  Socher, and Pascale Fung. 2019.
\newblock \href {https://doi.org/10.18653/v1/P19-1078} {Transferable
  multi-domain state generator for task-oriented dialogue systems}.
\newblock In \emph{Proceedings of the 57th Annual Meeting of the Association
  for Computational Linguistics}, pages 808--819, Florence, Italy. Association
  for Computational Linguistics.

\bibitem[{Ye et~al.(2021{\natexlab{a}})Ye, Manotumruksa, and
  Yilmaz}]{ye2021multiwoz}
Fanghua Ye, Jarana Manotumruksa, and Emine Yilmaz. 2021{\natexlab{a}}.
\newblock Multiwoz 2.4: A multi-domain task-oriented dialogue dataset with
  essential annotation corrections to improve state tracking evaluation.
\newblock \emph{arXiv preprint arXiv:2104.00773}.

\bibitem[{Ye et~al.(2021{\natexlab{b}})Ye, Manotumruksa, Zhang, Li, and
  Yilmaz}]{ye2021slot}
Fanghua Ye, Jarana Manotumruksa, Qiang Zhang, Shenghui Li, and Emine Yilmaz.
  2021{\natexlab{b}}.
\newblock Slot self-attentive dialogue state tracking.
\newblock In \emph{Proceedings of the Web Conference 2021}, pages 1598--1608.

\bibitem[{Young et~al.(2013)Young, Ga{\v{s}}i{\'c}, Thomson, and
  Williams}]{young2013pomdp}
Steve Young, Milica Ga{\v{s}}i{\'c}, Blaise Thomson, and Jason~D Williams.
  2013.
\newblock Pomdp-based statistical spoken dialog systems: A review.
\newblock \emph{Proceedings of the IEEE}, 101(5):1160--1179.

\bibitem[{Zang et~al.(2020)Zang, Rastogi, Sunkara, Gupta, Zhang, and
  Chen}]{zang2020multiwoz}
Xiaoxue Zang, Abhinav Rastogi, Srinivas Sunkara, Raghav Gupta, Jianguo Zhang,
  and Jindong Chen. 2020.
\newblock Multiwoz 2.2: A dialogue dataset with additional annotation
  corrections and state tracking baselines.
\newblock \emph{arXiv preprint arXiv:2007.12720}.

\bibitem[{Zhang et~al.(2016)Zhang, Bengio, Hardt, Recht, and
  Vinyals}]{zhang2016understanding}
Chiyuan Zhang, Samy Bengio, Moritz Hardt, Benjamin Recht, and Oriol Vinyals.
  2016.
\newblock Understanding deep learning requires rethinking generalization.
\newblock \emph{arXiv preprint arXiv:1611.03530}.

\bibitem[{Zhang et~al.(2020)Zhang, Hashimoto, Wu, Wang, Yu, Socher, and
  Xiong}]{zhang-etal-2020-find}
Jianguo Zhang, Kazuma Hashimoto, Chien-Sheng Wu, Yao Wang, Philip Yu, Richard
  Socher, and Caiming Xiong. 2020.
\newblock \href {https://www.aclweb.org/anthology/2020.starsem-1.17} {Find or
  classify? dual strategy for slot-value predictions on multi-domain dialog
  state tracking}.
\newblock In \emph{Proceedings of the Ninth Joint Conference on Lexical and
  Computational Semantics}, pages 154--167, Barcelona, Spain (Online).
  Association for Computational Linguistics.

\bibitem[{Zhang and Sabuncu(2018)}]{zhang2018generalized}
Zhilu Zhang and Mert~R Sabuncu. 2018.
\newblock Generalized cross entropy loss for training deep neural networks with
  noisy labels.
\newblock In \emph{Proceedings of the 32nd International Conference on Neural
  Information Processing Systems}, pages 8792--8802.

\bibitem[{Zhu et~al.(2020)Zhu, Huang, Zhang, Zhu, and Huang}]{zhu2020crosswoz}
Qi~Zhu, Kaili Huang, Zheng Zhang, Xiaoyan Zhu, and Minlie Huang. 2020.
\newblock Crosswoz: A large-scale chinese cross-domain task-oriented dialogue
  dataset.
\newblock \emph{Transactions of the Association for Computational Linguistics},
  8:281--295.

\end{thebibliography}
\bibliographystyle{acl_natbib}

\appendix

\section{Proof of Theorem 1}
\label{sec:appendixproof}
\begin{proof}
Our proof is based on the bias-variance decomposition theorem\footnote{\url{https://en.wikipedia.org/wiki/Bias-variance_tradeoff}}. For any sample $\mathcal{X}_t$ in the noisy training set $\mathcal{D}_n$, the approximation error with respect to the pseudo label $\breve{v}_t$ of slot $s$ is defined as $E_{\mathcal{D}_c}[\Vert \breve{\bm{v}}_t - \bm{v}_t \Vert^2_2]$, which, according to the bias-variance decomposition theorem, can be decomposed into a bias term and a variance term, i.e., 
\begin{equation*}
    E_{\mathcal{D}_c}[\Vert \breve{\bm{v}}_t - \bm{v}_t \Vert^2_2] = (\texttt{Bias}_{\mathcal{D}_c}[\breve{\bm{v}}_t])^2 + \texttt{Var}_{\mathcal{D}_c} [\breve{\bm{v}}_t],
\end{equation*}
where
\begin{equation*}
\begin{aligned}
    &\texttt{Bias}_{\mathcal{D}_c}[\breve{\bm{v}}_t] = \Vert E_{\mathcal{D}_c}[\breve{\bm{v}}_t] - \bm{v}_t \Vert_2, \\
    &\texttt{Var}_{\mathcal{D}_c} [\breve{\bm{v}}_t] = E_{\mathcal{D}_c}[\Vert E_{\mathcal{D}_c}[\breve{\bm{v}}_t] - \breve{\bm{v}}_t \Vert^2_2].
\end{aligned}
\end{equation*}

In our approach, the auxiliary model is a BERT-based model, which has more than 110M parameters. Such a complex model is expected to be able to capture all the samples in the small clean dataset $\mathcal{D}_c$. Therefore, we can reasonably assume that the bias term is close to zero. Then, we have:
\begin{equation*}
\begin{aligned}
    \texttt{Bias}_{\mathcal{D}_c}[\breve{\bm{v}}_t] \approx 0 \Rightarrow E_{\mathcal{D}_c}[\breve{\bm{v}}_t] \approx \bm{v}_t. 
\end{aligned}
\end{equation*}

Considering that the pseudo labels are generated by the auxiliary model that is trained on an extra small clean dataset and this clean dataset is independent of the noisy training set, we can regard the pseudo labels and vanilla labels as independent of each other. Consequently, we obtain:
\begin{equation*}
    \begin{aligned}
        E_{\mathcal{D}_c}&[(\tilde{\bm{v}}_t - \bm{v}_t)^T(\breve{\bm{v}}_t - \bm{v}_t)]  \\
        &= [ E_{\mathcal{D}_c}[\tilde{\bm{v}}_t - \bm{v}_t]]^T E_{\mathcal{D}_c}[\breve{\bm{v}}_t - \bm{v}_t] \\
        &= [ E_{\mathcal{D}_c}[\tilde{\bm{v}}_t - \bm{v}_t]]^T E_{\mathcal{D}_c}[\breve{\bm{v}}_t - E_{\mathcal{D}_c}[\breve{\bm{v}}_t]] \\
        & = [ E_{\mathcal{D}_c}[\tilde{\bm{v}}_t - \bm{v}_t]]^T \bm{0} = 0.
    \end{aligned}
\end{equation*}

Based on the formula above, we can now calculate the approximation error with respect to the combined label $\bm{v}^c_t$ of slot $s$ as below:
\begin{equation*}
    \begin{aligned}
        E_{\mathcal{D}_c}&[\Vert \bm{v}^c_t - \bm{v}_t \Vert^2_2] \\
        &= E_{\mathcal{D}_c}[\Vert \alpha \breve{\bm{v}}_t + (1-\alpha)\tilde{\bm{v}}_t - \bm{v}_t \Vert^2_2] \\
        &= E_{\mathcal{D}_c}[\Vert \alpha (\breve{\bm{v}}_t - \bm{v}_t) + (1-\alpha)(\tilde{\bm{v}}_t - \bm{v}_t) \Vert^2_2]\\
        &= \alpha^2  E_{\mathcal{D}_c}[\Vert \breve{\bm{v}}_t - \bm{v}_t \Vert^2_2] \\
        &\quad+ (1-\alpha)^2 E_{\mathcal{D}_c}[\Vert \tilde{\bm{v}}_t - \bm{v}_t \Vert^2_2],
    \end{aligned}
\end{equation*}
where the last equality holds because of $E_{\mathcal{D}_c}[(\tilde{\bm{v}}_t - \bm{v}_t)^T(\breve{\bm{v}}_t - \bm{v}_t)] = 0$.
Then, we have:
\begin{equation*}
    \begin{aligned}
        Y_{{\bm{v}}^c} &= \frac{1}{|\mathcal{D}_n||\mathcal{S}|}\sum_{\mathcal{X}_t \in \mathcal{D}_n} \sum_{s \in \mathcal{S}} E_{\mathcal{D}_c} [ \Vert \bm{v}^c_t - \bm{v}_t \Vert^2_2] \\
        &=  \frac{\alpha^2}{|\mathcal{D}_n||\mathcal{S}|}\sum_{\mathcal{X}_t \in \mathcal{D}_n} \sum_{s \in \mathcal{S}} E_{\mathcal{D}_c} [\Vert \breve{\bm{v}}_t - \bm{v}_t \Vert^2_2] \\
        &\quad+  \frac{(1-\alpha)^2}{|\mathcal{D}_n||\mathcal{S}|}\sum_{\mathcal{X}_t \in \mathcal{D}_n} \sum_{s \in \mathcal{S}} E_{\mathcal{D}_c}[\Vert \tilde{\bm{v}}_t - \bm{v}_t \Vert^2_2] \\
        &= \alpha^2 Y_{\breve{\bm{v}}} + (1-\alpha)^2 Y_{\tilde{\bm{v}}}.
    \end{aligned}
\end{equation*}
$Y_{{\bm{v}}^c}$ reaches its minimum when $\alpha = \frac{Y_{\tilde{\bm{v}}}}{Y_{\tilde{\bm{v}}} + Y_{\breve{\bm{v}}}}$, and 
\begin{equation*}
    \min_{\alpha} Y_{{\bm{v}}^c} = \frac{Y_{\tilde{\bm{v}}} Y_{\breve{\bm{v}}}}{Y_{\tilde{\bm{v}}} + Y_{\breve{\bm{v}}}},
\end{equation*}
which concludes the proof.
\end{proof}

\section{Training Details}
\label{sec:appendixtrain}
Note that the proposed auxiliary model is also applied as one primary model in our experiments. In both cases, AdamW~\citep{kingma2014adam} is adopted as the optimizer, and a linear schedule with warmup is created to adjust the learning rate dynamically. The peak learning rate is set to 2.5e-5. The warmup proportion is fixed at 0.1. The dropout~\citep{srivastava2014dropout} probability and word dropout~\citep{bowman-etal-2016-generating} probability are also fixed at 0.1. When taken as the auxiliary model, the model is trained for at most 30 epochs with a batch size of 8. When taken as the primary model, the batch size and training epochs are set to 8 and 12, respectively. The best model is chosen according to the performance on the validation set. We apply left truncation when the input exceeds the maximum input length of BERT.

For SOM-DST and STAR, the default hyperparameters  are adopted when they are applied as the primary model (except setting \textit{num\_workers = 0}).

\section{Additional Experimental Results}

\begin{figure}[h!]
  \centering
  \includegraphics[width=0.9\linewidth]{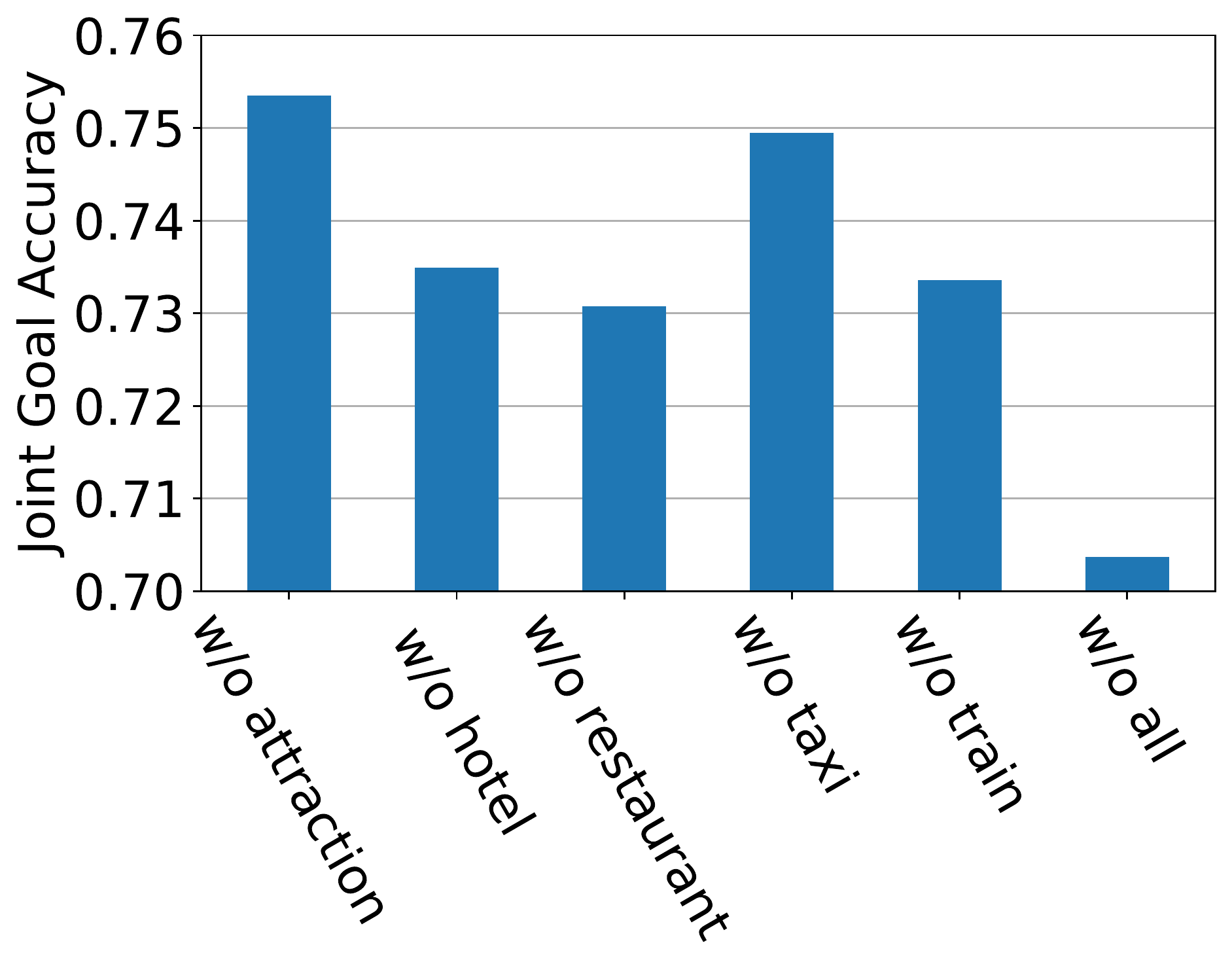}
  \caption{Analyses on the effects of the distribution of the clean dataset by removing all the dialogues related to each domain. "w/o all" means no clean data is used.}
  \label{fig:varydist}
\end{figure}

\subsection{Effects of the Distribution of the Clean Dataset}

Except for the size of the clean dataset, the distribution of the clean dataset may also affect the performance of the primary model, especially when the clean dataset has a significantly different distribution from the training set. Thus, it is important to study the effects of the distribution of the clean dataset. However, we are short of clean datasets with different distributions. It is also challenging to model the distribution explicitly since the dialogue state may contain multiple labels. To address this issue, we propose to remove all the dialogues that are related to a specific domain and use only the remaining ones as the clean dataset. As thus, we can create multiple clean datasets with different distributions. The results of AUX-DST on MultiWOZ 2.4 are shown in Figure~\ref{fig:varydist}. As can be observed, although different clean datasets indeed lead to different performance, compared to the situation where no clean data is used (i.e., only the vanilla labels are used to train the model), all these clean datasets still bring huge performance improvements.

\end{document}